\RequirePackage{etoolbox}

\PassOptionsToPackage{hypertexnames=false}{hyperref}

\documentclass{article}

\usepackage[preprint]{neurips_2022} %

\usepackage[utf8]{inputenc} %
\usepackage[T1]{fontenc}    %
\usepackage{hyperref}       %
\usepackage{url}            %
\usepackage{booktabs}       %
\usepackage{amsfonts}       %
\usepackage{nicefrac}       %
\usepackage{microtype}      %
\usepackage{xcolor}         %

\usepackage{amsthm}

\usepackage{amsmath}
\usepackage{cases}
\usepackage[capitalise]{cleveref}

\usepackage{tikz}
\usetikzlibrary{automata, arrows.meta, positioning}
\usepackage{sidecap}
\sidecaptionvpos{figure}{m}

 \newcommand{\Probab}{\P} %
 
 \newcommand{\ordot}{\tilde{\ordo}}
 \newcommand{\thetat}{\tilde{\Theta}}
 \newcommand{\ordo}{\mathcal{O}}

 \newcommand{\I}{{\mathbb{I}}}

 \newcommand{\tphi}{\tilde\phi}

\newcommand{\lmin}{{\ell}}

\newcommand{\ubold}{\hat q(s,\pi(s)) + \omega}
\newcommand{\lbnew}{ \maxa \hat q(s,a) - \omega}
\newcommand{\uboldpi}{\hat q^{\pi}(s,\pi(s)) + \omega}
\newcommand{\lbnewpi}{ \maxa \hat q^{\pi}(s,a) - \omega}
\newcommand{\maxa}{\max_{a\in\cA}}
\newcommand{\argmaxa}{\argmax_{a\in\cA}}

\newcommand{\Sfix}{\cS_{\text{fix}}}
\newcommand{\Snotfix}{\cS\setminus\cS_{\text{fix}}}

\newcommand{\Pidet}{\Pi_\text{det}}
\newcommand{\extend}[2]{\Pi_{{#1},{#2}}}
\newcommand{\estimate}{\textsc{Estimate}}
\newcommand{\simulator}{\textsc{Simulator}}
\newcommand{\cover}{\text{Cover}}
\newcommand{\actioncover}{\text{ActionCover}}
\newcommand{\measure}{\textsc{Measure}}
\newcommand{\lse}{\text{LSE}}
\newcommand{\Hgamma}{{H}}
\newcommand{\mge}{\succeq}

\newcommand{\thetabound}{B}
\newcommand{\featurebound}{L}

\newcommand{\unknown}{{\bot}}

\newcommand{\red}[1]{{\color{red}{#1}}}
\newcommand{\blue}[1]{{\color{blue}{#1}}}
\newcommand{\diff}{\mathop{}\!\mathrm{d}}

\newcommand{\sign}[1]{\mathrm{sgn}(#1)}
\newcommand{\err}{\mathrm{err}}
\newcommand{\KL}{\mathrm{KL}}

 \usepackage[noend]{algpseudocode}
\usepackage{algorithmicx}
\usepackage{varwidth}

\usepackage{bbm}

\algnewcommand{\algorithmicgoto}{\textbf{goto}}%
\algnewcommand{\Goto}[1]{\algorithmicgoto~\ref{#1}}%
\algnewcommand{\Break}{\textbf{break}}%

\algnewcommand{\Initialize}[1]{%
  \State \textbf{Initialize:}
  \Statex \hspace*{\algorithmicindent}\parbox[t]{.8\linewidth}{\raggedright #1}
}
\algnewcommand{\Inputs}[1]{%
  \State \textbf{Inputs:}
  \Statex \hspace*{\algorithmicindent}\parbox[t]{.8\linewidth}{\raggedright #1}
}

\usepackage[utf8]{inputenc}
\usepackage[english]{babel}
\usepackage{latexsym}
\usepackage{mathrsfs}
\usepackage{mathtools}
\usepackage{bm}
\usepackage{datetime}
\usepackage{tikz}
\usepackage{listings}
\usepackage{mdframed}
\usepackage{pgfplots}
\usepackage{pgfplotstable}
\usepackage{dsfont}
\usepackage{color}
\usepackage{colortbl}
\usepackage{pifont}
\usepackage[bf,font=small,singlelinecheck=off]{caption}

\usepackage[textsize=tiny,disable]{todonotes} %

\usepackage{microtype} %
\usepackage{float}
\usepackage{xfrac} %
\usepackage{xspace}

\renewcommand{\phi}{\varphi}

\usepackage{nicefrac}
\usepackage{wrapfig}

\usepackage[normalem]{ulem}

\newcommand{\M}{\mathcal{M}}
\newcommand{\tM}{\tilde{M}}
\newcommand{\tiM}{\mathcal{\tM}}

\renewcommand{\P}{\mathcal{P}}

\DeclareMathOperator{\Dists}{\mathcal{M}_1}

\definecolor{emerald}{rgb}{0.31, 0.78, 0.47}

\usepackage{enumerate}
\newtheorem{theorem}{Theorem}[section]

 \newtheorem{lemma}[theorem]{Lemma}
 
 \newtheorem{definition}[theorem]{Definition}
\newtheorem{assumption}[theorem]{Assumption}

\usepackage{aliascnt}

\usepackage{xspace}

\usepackage{enumerate}

\newcommand{\E}{\mathbb E}

\newcommand{\ip}[1]{\left\langle #1 \right\rangle}

\newcommand{\norm}[1]{\left\|#1\right\|}
\newcommand{\R}{\mathbb{R}}
\newcommand{\N}{\mathbb{N}}
\newcommand{\cA}{\mathcal{A}}
\newcommand{\cB}{\mathcal{B}}

\newcommand{\cM}{\mathcal{M}}
\newcommand{\cN}{\mathcal{N}}

\newcommand{\cP}{\mathcal{P}}
\newcommand{\cQ}{\mathcal{Q}}
\newcommand{\cR}{\mathcal{R}}
\newcommand{\cS}{\mathcal{S}}

\renewcommand{\epsilon}{\varepsilon}
\newcommand{\ceil}[1]{\left\lceil {#1} \right\rceil}
\newcommand{\floor}[1]{\left\lfloor {#1} \right\rfloor}

\DeclareMathOperator*{\argmax}{arg\ max}

\newif\ifsup\suptrue

\usepackage{algorithm}
\usepackage{algpseudocode}

\makeatletter
 \let\Ginclude@graphics\@org@Ginclude@graphics 
\makeatother

\usepackage{times}
\usepackage{newtxmath}

\usepackage{thmtools}
\usepackage{thm-restate}

\usepackage[utf8]{inputenc}

\title{Confident Approximate Policy Iteration for Efficient Local Planning in $q^\pi$-realizable MDPs}

\author{%
 Gell\'ert Weisz\\
 DeepMind, London, UK\\
 University College London, London, UK\\
 \And
 Andr\'as {Gy}\"orgy\\
 DeepMind, London, UK
 \And
 Tadashi Kozuno\\
 University of Alberta, Edmonton, Canada\\
 Omron Sinic X, Tokyo, Japan\\
 \And
 {Cs}aba {Sz}epesv\'ari\\
 DeepMind, London, UK\\
 University of Alberta, Edmonton, Canada\\
}

\begin{document}

\maketitle
\newif\ifchecklist
\checklistfalse

\begin{abstract}
We consider approximate dynamic programming in $\gamma$-discounted Markov decision processes
and apply it to approximate planning with linear value-function approximation.
Our first contribution is a new 
variant of \textsc{Approximate Policy Iteration} (\textsc{API}), called %
\textsc{Confident Approximate Policy Iteration} (\textsc{CAPI}),
which computes
a deterministic stationary policy with an optimal error bound scaling linearly with 
the product of 
the effective horizon $H$
and the worst-case approximation error  $\epsilon$ of the action-value functions %
of stationary policies. 
 This improvement over \textsc{API} (whose error scales with $H^2$) comes at the price of an $H$-fold increase in memory cost.
Unlike \citet{scherrer2012}, who recommended computing a non-stationary policy to achieve a similar improvement (with the same memory overhead), we are able to stick to stationary policies. 
This allows for our second contribution, %
the application of CAPI to planning with local access to a simulator and $d$-dimensional linear function approximation.
As such, we design a planning algorithm that applies CAPI to obtain a sequence of policies with successively refined accuracies 
on a dynamically evolving set of states. 
The algorithm outputs an $\ordot(\sqrt{d}H\epsilon)$-optimal policy after issuing $\ordot(dH^4/\epsilon^2)$ queries to the simulator, simultaneously achieving the optimal accuracy bound and the best known query complexity bound, %
while earlier algorithms in the literature %
achieve only one of them. %
This query complexity is shown to be tight in all parameters except $H$.
These improvements come at the expense of a mild (polynomial) increase in memory and computational costs of both the algorithm and %
its output policy.
\end{abstract}

\section{Introduction}\label{sec:intro}

A key question in reinforcement learning is how to use value-function approximation to arrive at scaleable 
algorithms that can find near-optimal policies in Markov decision processes (MDPs).
A flurry of recent results aims at solving this problem efficiently with varying models of interaction with the MDP.
In this paper we focus on the problem of \emph{planning with a simulator} when using linear function approximation.
A simulator is a ``device'' that, given a state-action pair as a query, returns a next state and reward generated from the transition kernel of the MDP that is simulated.
Depending on the application, such a simulator is often readily available (e.g., in chess, go, Atari). %
Planning with simulator access comes with great benefits:
for example, in a recent work, \citet{wang2021exponential} showed that under some conditions it is exponentially more efficient to find a near-optimal policy if a simulator of the MDP (that can reset to a state) is available compared to the online case where a learner interacts with its environment by following trajectories but without the help of a simulator.

Our setting of \emph{offline, local planning}
considers the problem of finding a policy with near-optimal value at a given initial state $s_0$ in the MDP.
The planner can issue queries to the simulator, and has to find and output a near-optimal policy with high probability. %
The efficiency of a planner is measured in four ways: the \emph{suboptimality} of the policy found, that is, how far its value is from that of the optimal policy; the \emph{query cost},  that is, the number of queries issued to the simulator;  the \emph{computational cost}, which is the number of operations used; and the \emph{memory cost}, which is the amount of memory used
(we adopt the real computation model for these costs).
There are several interaction models between the planner and the simulator \citep{yin2022efficient}.
The most permissive one is called
the \emph{generative model}, or \emph{random access}.
Here, the planner receives the set of all states and is allowed to issue queries for any state and action.
Coding a simulator that supports this model can be challenging, as oftentimes the set of states is computationally difficult to describe.
Instead of \emph{random access}, we consider the more practical and more challenging \emph{local access} setting,
where the planner only sees the initial state and the set of states received as a result to a query to the simulator.
Consequently, the queries issued have to be for a state that has already been encountered this way (and any available action), while the simulator needs to support the ability to 
reset the MDP state \emph{only} to previously seen states. A simple approach in practice to support this model is saving and reloading checkpoints during the operation of the simulator.

To handle large, possibly infinite state spaces, we use linear function approximation to approximate the action-value functions $q^\pi$ of stationary, deterministic policies $\pi$ (for background on MDPs, see the next section). 
A feature-map is a good fit to an MDP if the worst-case error of using the feature-map to approximate value functions of policies of the MDP is small:
\begin{definition}[$q^\pi$-realizability: uniform policy value-function approximation error]
\label{ass:q-pi-realizability}
\label{def:q-pi-realizable}
Given an MDP, the uniform policy value-function approximation error induced by 
a feature map $\phi$, which maps state-action pairs $(s,a)$ to the Euclidean ball of radius $\featurebound$ centered at zero in $\R^d$, 
over a set of parameters belonging to the $d$-dimensional centered Euclidean ball of radius $\thetabound$ is 
\[
\epsilon = \sup_{\pi} \inf_{\theta:\|\theta\|_2 \le \thetabound} \sup_{(s,a)} | q^\pi(s,a) - \ip{\phi(s,a),\theta} |\,,
\]
where the outermost supremum is over all possible stationary deterministic memoryless policies (i.e., maps from states to actions) of the MDP. 
\end{definition}
Our goal is to design algorithms that scale gracefully with the uniform approximation error $\epsilon$ at the expense of controlled computational cost. 
To achieve nontrivial guarantees, the uniform approximation error needs to be ``small''.
This (implicit) assumption is stronger than the $q^\star$-realizability assumption (where the approximation error is only considered for optimal policies), which
\citet{weisz2020exponential} showed an exponential query complexity lower bound for.
At the same time, it is (strictly) weaker than the linear MDP assumption \citep{zanette2020learning}, for which there are
 efficient algorithms to find a near-optimal policy in the online setting (without a simulator) \citep{jin2020provably}, even in the more challenging reward-free setting where the rewards are only revealed after exploration \citep{wagenmaker2022reward}. 

In the \emph{local access} setting, the planner learns the features $\phi(s,a)$ of a state-action pair \emph{only} for those states $s$ that have already been encountered.
In contrast, in the \emph{random access} setting, the whole feature map $\phi(\cdot,\cdot)$, of (possibly infinite) size $d |\cS||\cA|$ (where $\cS$ and $\cA$ are the state and action sets, resp.), is given to the planner as input.
In the latter setting, when only the query cost is counted, \citet{Du_Kakade_Wang_Yan_2019} and \citet{LaSzeGe19} proposed algorithms (the latter working in the misspecified, $\epsilon>0$ regime) that issue a number of queries that is polynomial in the relevant parameters, but require 
a barycentric spanner or near-optimal design of the input features.
In the worst case, computing any of these sets scales polynomially in $|\cS|$ and $|\cA|$, which can be prohibitive.

In the case of \emph{local access}, considered in this paper, 
the best known bound on the suboptimality of the computed policy is achieved by
\textsc{Confident MC-Politex}
\citep{yin2022efficient}.
In the more permissive \emph{random access} setting, the best known query cost is achieved by \cite{LaSzeGe19}.
Our algorithm, \textsc{CAPI-Qpi-Plan} (given in \cref{alg:capi-qpi-plan}), %
achieves the \emph{best of both} while only assuming \emph{local access}.
This is shown in the next theorem; in the theorem $\epsilon$ is as defined in \cref{def:q-pi-realizable}, $\gamma$ is the discount factor, and $v^\star$ and $v^\pi$ are the state value functions associated with the optimal policy and policy $\pi$, respectively (precise definitions of these quantities are given in the next section).
A comparison to other algorithms in the literature is given in \cref{table:comparsion-of-guarantees}; there the accuracy parameter $\omega$ of the algorithms is set to $\epsilon$, but a larger $\omega$ can be used to trade off suboptimality guarantees for an improved query cost.

\begin{theorem}\label{thm:qplan-main}
For any confidence parameter $\delta\in(0,1]$, accuracy parameter $\omega>0$, and initial state $s_0\in\cS$,
with probability at least $1-\delta$, \textsc{CAPI-Qpi-Plan}  (\cref{alg:capi-qpi-plan}) %
finds a policy $\pi$ with %
\begin{equation}
\label{eq:thm-accuracy}
v^\star(s_0)-v^\pi(s_0) = \ordot\left((\epsilon+\omega)\sqrt{d}(1-\gamma)^{-1}\right)\,,
\end{equation}
while executing at most
$\ordot\left(d (1-\gamma)^{-4}\omega^{-2}\right)$
queries in the \emph{local access} setting.
\end{theorem}
\textsc{CAPI-Qpi-Plan} is based on \textsc{Confident MC-LSPI}, another algorithm of \citet{yin2022efficient}, which relies on policy iteration from a \emph{core set} of informative state-action pairs, but achieves inferior performance both in terms of suboptimality and query complexity.
However, \textsc{CAPI-Qpi-Plan}'s improvements come at the expense of increased memory and computational costs, as shown in the next theorem:
compared to \textsc{Confident MC-LSPI}, the memory and computational costs of our algorithm increase by a factor of the effective horizon $H=\ordot(1/(1-\gamma))$, and the policy computed by \textsc{CAPI-Qpi-Plan} uses a $dH$ factor more memory for storage and a $d^2H$ factor more computation to evaluate. 

\begin{theorem}[Memory and computational cost]\label{thm:mem-comp-cost}
The memory and computational cost of running \textsc{CAPI-Qpi-Plan} (\cref{alg:capi-qpi-plan}) are
$\ordot\left(d^2/(1-\gamma)\right)$ and $\ordot\left(d^4|\cA|(1-\gamma)^{-5}\omega^{-2}\right)$,
respectively, while the memory and computational costs of storing and evaluating the final policy outputted by \textsc{CAPI-Qpi-Plan}, respectively, are
$\ordot\left(d^2/(1-\gamma)\right)$ and $\ordot\left(d^3|\cA|/(1-\gamma)\right)$.
\end{theorem}

Next we present a lower bound corresponding to \cref{thm:qplan-main}
that holds even in the more permissive \emph{random access} setting, and shows that \textsc{CAPI-Qpi-Plan} trades of the query cost and the suboptimality of the returned policy  asymptotically optimally up to its dependence on $1/(1-\gamma)$:
\begin{theorem}[Query cost lower bound]\label{thm:main-lower}
Let %
$\alpha\in(0,\frac{0.05\gamma}{(1-\gamma)(1+\gamma)^2})$, 
$\delta\in(0,0.08]$, 
$\gamma\in[\frac{7}{12},1]$, $d\ge3$, and $\epsilon\ge0$.
Then there is a class $\M$ of MDPs with uniform policy value-function approximation error at most $\epsilon$
such that any planner that guarantees to find an $\alpha$-optimal policy $\pi$ (i.e., $v^\star(s_0)-v^\pi(s_0) \le \alpha$)
with probability at least $1-\delta$ for all $M\in\M$ when used with a simulator for $M$ with \emph{random access},
the worst-case (over $\M$) expected number of queries issued by the planner is at least
\begin{align}\label{eq:lb-main}
\max\left(\exp\Big(\Omega\Big(\frac{d\epsilon^2}{\alpha^2(1-\gamma)^2}\Big)\Big),\, \Omega \left( \frac{d^2}{\alpha^2(1-\gamma)^3} \right)\right)\,.
\end{align}
\end{theorem}
If $\omega$ is set to $\epsilon$ for \textsc{CAPI-Qpi-Plan},
the first term of \cref{eq:lb-main} implies that any planner with an asymptotically smaller (apart from logarithmic factors) suboptimality guarantee than \cref{eq:thm-accuracy}
executes exponentially many queries in expectation. %
The second term of \cref{eq:lb-main}, which is shown to be a lower bound in \cref{theorem:high-prob lower bound} even in the more general setting of
linear MDPs with zero misspecification ($\epsilon=0$), matches the query complexity of \cref{thm:qplan-main} up to an $\ordot((1-\gamma)^2)$ factor.
Thus, the lower bound implies
that the suboptimality and query cost bounds of \cref{thm:qplan-main} are tight up to logarithmic factors in all parameters except the $1/(1-\gamma)$-dependence of the query cost bound.

At the heart of our method is a new algorithm, which we call \textsc{Confident Approximate Policy Iteration} (\textsc{CAPI}). 
This algorithm, which belongs to the family of approximate dynamic programming algorithms \citep{Ber12,Munos03,Munos05}, is a novel variant of \textsc{Approximate Policy Iteration} (\textsc{API}) \citep{BeTs96}:
in the policy improvement step, \textsc{CAPI} only updates the policy in states where it is confident that the update will improve the performance. 
This simple modification allows \textsc{CAPI} to avoid the problem of ``classical'' approximate dynamic programming algorithms (approximate policy and value iteration) 
of inflating the value function evaluation error by a factor of $H^2$ where $H=\ordot(1/(1-\gamma))$ 
(for discussions of this problem, see also the papers by \citealp{scherrer2012} and \citealp{russo2020approximation}), and reduce this inflation factor to $H$.
A similar result has already been achieved by \citet{scherrer2012}, who proposed  to construct a non-stationary policy that strings together all policies obtained while running either approximate value or policy iteration. %
However, applying this result to our planning problem is problematic, since the policies to be evaluated are non-stationary, and hence including them in the policy set we aim to approximate may drastically increase the error $\epsilon$ as compared to \cref{def:q-pi-realizable}, which only considers stationary memoryless policies.

While the improvements provided by \textsc{CAPI} allows \textsc{CAPI-Qpi-Plan} to match the performance of \textsc{Confident MC-Politex} in terms of suboptimality, 
it is unlikely that a simple modification of \textsc{Confident MC-Politex} would lead to an algorithm which matches 
\textsc{CAPI-Qpi-Plan}'s performance in terms of query cost (see \cref{table:comparsion-of-guarantees}): 
Both methods evaluate a sequence of policies at an $\ordot(\epsilon)$ accuracy each (requiring $\ordot(1/\epsilon^{2})$ queries, 
omitting the dependence on other parameters). 
However, while \textsc{CAPI-Qpi-Plan}  (and \textsc{Confident MC-LSPI}) evaluates $\ordo(\log(1/\epsilon))$ (again in terms of $\epsilon$ only) policies to find one which is $\ordot(\epsilon)$-optimal, 
\textsc{Confident MC-Politex} needs to compute $\ordot(1/\epsilon^2)$ policies to achieve the same.
As a consequence, \textsc{Confident MC-Politex} only achieves $\ordot(1/\epsilon^4)$ query complexity, and to match  \textsc{CAPI-Qpi-Plan}'s $\ordot(1/\epsilon^2)$ complexity, one would need to come up with either significantly better policy evaluation methods 
(potentially using the similarity in the subsequent policies) or a much faster 
(exponential vs. square-root) convergence rate in the suboptimality of the policy sequence produced by \textsc{Confident MC-Politex}.

The rest of the paper is organized as follows: The model and notation are introduced in \cref{sec:prelim}.  \textsc{CAPI} is introduced and analyzed in \cref{sec:capi}.
Planning with $q^\pi$-realizability is introduced in \cref{sec:planning}, with \textsc{CAPI-Qpi-Plan} being built-up and analyzed in \cref{sec:estimate,sec:mainalg}. In particular, the proof of \cref{thm:qplan-main} is given in \cref{sec:mainalg}. Several proofs are relegated to appendices, in particular, \cref{thm:mem-comp-cost} is proved and implementation details of \textsc{CAPI-Qpi-Plan} are discussed in \cref{app:memcomp-impl-thm-proof}, while \cref{thm:main-lower} is proved in \cref{app:lower-bounds}.

\begin{table}[t]
\centering
\captionsetup{justification=centering}
\vspace{-3mm}
\begin{center}
\caption{Comparison of suboptimality and query complexity guarantees of various planners (with the approximation accuracy parameter $\omega$ set to $\epsilon$). Drawbacks are highlighted with \red{red}, the best bounds with \blue{blue}.
\label{table:comparsion-of-guarantees}
}
\end{center}
\begin{tabular}{cccc} 
 \toprule
 Algorithm [Publication] & Query cost & Suboptimality & Access model \\ 
 \midrule
MC-LSPI
 \citep{LaSzeGe19} & $\blue{\ordot\big(\frac{d}{\epsilon^2(1-\gamma)^4}\big)}$ & $\ordot\big(\frac{\epsilon\sqrt{d}}{(1-\gamma)^\red{2}}\big)$ & \red{random} access \\ 
 \hline
 \textsc{Confident MC-LSPI} \citep{yin2022efficient} & $\ordot\big(\frac{d^\red{2}}{\epsilon^2(1-\gamma)^4}\big)$ & $\ordot\big(\frac{\epsilon\sqrt{d}}{(1-\gamma)^\red{2}}\big)$ & local access \\ 
 \hline
 \textsc{Confident MC-Politex} \citep{yin2022efficient} & $\ordot\big(\frac{d}{\epsilon^\red{4}(1-\gamma)^\red{5}}\big)$ & $\blue{\ordot\big(\frac{\epsilon\sqrt{d}}{1-\gamma}\big)}$ & local access \\
 \hline
 \textsc{CAPI-Qpi-Plan} [This work] & $\blue{\ordot\big(\frac{d}{\epsilon^2(1-\gamma)^4}\big)}$ & $\blue{\ordot\big(\frac{\epsilon\sqrt{d}}{1-\gamma}\big)}$ & local access \\
 \bottomrule
\end{tabular} %
\vspace{-3mm}
\end{table}

\section{Notation and preliminaries}
\label{sec:prelim}

Let $\N=\{0, 1, \ldots\}$ denote the set of natural numbers, $\N^+=\{1, 2, \ldots\}$ the positive integers.
For some integer $i$, let $[i]=\{0, \ldots, i-1\}$. %
For $x\in\R$, let $\ceil{x}$ denote the smallest integer i such that $i\ge x$.
For a positive definite $V\in\R^{d\times d}$ and $x\in\R^d$, let $\norm{x}_V^2=x^\top Vx$.
For matrices $A$ and $B$, we say that $A \mge B$ if $A-B$ is positive semidefinite.
Let $\I$ be the $d$-dimensional identity matrix.
For compatible vectors $x,y$, let $\ip{x,y}$ be their inner product: $\ip{x,y}=x^\top y$.
Let $\Dists(X)$ denote the space of probability distributions supported on the set $X$
(throughout, we assume that the $\sigma$-algebra is implicit).
We write $a\approx_\epsilon b$ for $a,b,\epsilon\in\R$ if $|a-b|\le\epsilon$.
We denote by $\ordot(\cdot)$ and $\thetat(\cdot)$ the variants of the big-O notation that hide polylogarithmic factors.

A Markov Decision Process (MDP) is a tuple $M=(\cS, \cA, \cQ)$, where $\cS$ is a measurable state space, $\cA$ is a finite action space, and $\cQ:\cS\times \cA\to \Dists(\cS\times[0,1])$ is the transition-reward kernel.
We define the transition and reward distributions $P:\cS\times \cA\to \Dists(\cS)$ and $\cR:\cS\times\cA\to\Dists([0,1])$ as the marginals of $\cQ$.
By a slight abuse of notation, for any $s\in\cS$ and $a\in\cA$, let  $P(\cdot|s,a)$ and $\cR(\cdot|s,a)$ denote the distributions $P(s,a)$ and $\cR(s,a)$, respectively.
We further denote by $r(s,a)=\int_{0}^1 x \diff\cR(x|s,a)$ the expected reward for an action $a\in\cA$ taken in a state $s \in \cS$.
Without loss of generality, we assume that there is a designated initial state $s_0\in\cS$.

Starting from any state $s\in\cS$, a stationary memoryless policy $\pi:\cS\to\Dists(\cA)$ interacts with the MDP in a sequential manner for time-steps $t\in\N$, defining a probability distribution $\cP_{\pi,s}$ over the episode trajectory $\{S_i, A_i, R_i\}_{i\in\N}$ as follows:
$S_0=s$ deterministically,
$A_i\sim\pi(S_i)$, and $(S_{i+1},R_i)\sim \cQ(S_i,A_i)$.
By a slight variation, let $\cP_{\pi,s,a}$ denote (for some $a\in\cA$) the distribution of the trajectory when $A_0=a$ deterministically, while the distribution of the rest of the trajectory is defined analogously.

This allows us to conveniently define the expected state-value and action-value functions in the discounted setting we consider,
for some discount factor $0<\gamma<1$, respectively, as
\begin{align}
  v^\pi(s)=\E_{\pi,s}\left[\sum_{t\in\N} \gamma^t R_t\right] \label{eq:v-def}
  \quad\text{ and }\quad
  q^\pi(s,a)=\E_{\pi,s,a}\left[\sum_{t\in\N} \gamma^t R_t\right] \quad\quad\text{ for all } (s,a)\in\cS\times\cA\,,%
\end{align}
where throughout the paper we use the convention that $\E_{\bullet}$ is the expectation operator corresponding to a distribution $\P_{\bullet}$
(e.g., $\E_{\pi,s}$ is the expectation with respect to $\P_{\pi,s}$).
It is well known (see, e.g., \citealp{Put94}) that there exists an optimal stationary deterministic memoryless policy $\pi^\star$ such that
\[
\textstyle{\sup_\pi} v^\pi(s)=v^{\pi^\star}(s) \quad\quad\text{ and }\quad\quad
\textstyle{\sup_\pi} q^\pi(s,a)=q^{\pi^\star}(s,a) \quad\quad\text{ for all } (s,a)\in\cS\times\cA\,.
\]
Let $v^\star=v^{\pi^\star}$ and $q^\star=q^{\pi^\star}$.
For any policy $\pi$, $v^\pi$ and $q^\pi$ are known to satisfy the Bellman equations \citep{Put94}:
\begin{align}
\!\!\! v^\pi\!(s)\!=\!\! \sum_{a\in\cA}\!\pi(a|s) q^\pi\!(s,a) %
\!\,\,\,\text{and}\,\,\,\!
  q^\pi\!(s,a)&=r(s,a)+\gamma\!\!\!\int\displaylimits_{s'\in\cS}\!\!\! v^\pi\!(s')\diff P(s'|s,a)
\!\,\,\,\text{for all } (s,a)\!\in\!\cS\times\cA.\!
   \label{eq:bellman-q}
\\[-2em] \nonumber
\end{align}

Finally, we call a policy $\pi$ deterministic if for all states, $\pi(s)$ is a distribution that assigns unit weight to one action and zero weight to the others. 
With a slight abuse of notation, %
for a deterministic policy $\pi$, we denote by $\pi(s)$ 
the action $\pi$ chooses (deterministically) in state $s\in\cS$.

\section{Confident Approximate Policy Iteration}%
\label{sec:capi}

In this section we introduce \textsc{Confident Approximate Policy Iteration} (\textsc{CAPI}), our new approximate dynamic programming algorithm. In approximate dynamic programming, the methods are designed around oracles that return either an approximation to the application of the Bellman optimality operator to a value function (``approximate value iteration''), or an approximation to the value function of some policy (``approximate policy iteration''). Our setting is the second. The novelty is that we assume access to the accuracy of the approximation and use this knowledge to modify the policy update, which leads to improved guarantees on the suboptimality of the computed policy.

We present the pseudocodes of  
\textsc{API} \citep{BeTs96}
and \textsc{CAPI} jointly in \cref{alg:api-variants}:
starting from an arbitrary (deterministic) policy $\pi_0$, the algorithm
iterates a policy estimation (Line~\ref{line:api-magic-estimate}) and a policy update step (Line~\ref{line:api-update-line}) $I$ times.
The policy update for \textsc{API} is greedy with respect to the action-value estimates $\hat q$ and is defined as $\pi_{\hat q}(s)=\argmax_{a\in\cA} \hat q(s,a)$.
We assume that $\argmax_{a\in\cA}$ breaks ties in a consistent manner by ordering the actions (using the notation $\cA=(\cA_1,\ldots,\cA_{|\cA|})$) and always choosing action $\cA_i$ with the lowest index $i$ that achieves the maximum.
For \textsc{CAPI}, the policy update further relies on 
a global estimation-accuracy parameter $\omega$, and a set of fixed-states $\Sfix\subseteq \cS$.
For the purposes of this section, it is enough to keep $\Sfix=\{\}$.
\textsc{CAPI} updates the policy to one that acts greedily with respect to $\hat q$ \emph{only} on states that are not in $\Sfix$ and where it is confident that this leads to an improvement over the previous policy (Case~\ref{eq:capi-pi-update-2}); 
otherwise, the new policy will return the same action as the previous one (Case~\ref{eq:capi-pi-update-1}). %
To decide, $\ubold$ is treated as the upper bound on the previous policy's value, and $\lbnew$ as the lower bound of the action-value of the greedy action (Eq.~\ref{eq:capi-pi-update}):
\begin{subnumcases}{
  \pi_{\hat q, \pi, \Sfix}(s)=
  \label{eq:capi-pi-update}
}
  \argmax_{a\in\cA} \hat q(s,a)\,,
    &$\text{if } s\not\in\Sfix \text{ and } \ubold < \lbnew$\,;
      \label{eq:capi-pi-update-2}\\
  \pi(s)\,,
    &$\text{otherwise.}$
      \label{eq:capi-pi-update-1} 
\end{subnumcases}
Note that $\pi_{\hat q, \pi, \Sfix}$ also depends on $\omega$, however, this dependence is omitted from the notation (as $\omega$ is kept fixed throughout).

\textsc{CAPI} can also be seen as a refinement of \textsc{Conservative Policy Iteration} (\textsc{CPI}) of \citet{kakade2002approximately}
with some important differences:
While \textsc{CPI} introduces a global parameter to ensure the update stays close to the previous policy,
\textsc{CAPI} has no such parameter, and it dynamically decides when to stay close to (more precisely, use) the previous policy, individually for every state, based on whether there is evidence for a guaranteed improvement.

\begin{algorithm}[t]
\caption{\textsc{Approximate Policy Iteration} (API) and \textsc{Confident Approximate Policy Iteration} (CAPI) 
}\label{alg:api-variants}
\begin{algorithmic}[1]
\For{$i=1$ to $I$}
  \State $\hat q\gets\estimate(\pi_{i-1})$\label{line:api-magic-estimate}
  \State $\pi_i\gets \begin{cases} \label{line:api-capi}
\pi_{\hat q} & \textsc{API}\\
\pi_{\hat q, \pi_{i-1}, \Sfix}  & \textsc{CAPI}
\end{cases}$ \label{line:api-update-line} %
\EndFor
\State \Return $\pi_I$
\end{algorithmic}
\end{algorithm}

Let $\pi$ be any stationary deterministic memoryless policy, $\hat q^\pi: \cS\times\cA\to\R$ be any function, $\omega\in\R_+$, and $\Sfix\subseteq \cS$.
First, we show that as long as $\hat q^{\pi}$ is an $\omega$-accurate estimate of $q^{\pi}$, the \textsc{CAPI} policy update only improves the policy's values:

\begin{lemma}[No deterioration]\label{lem:no-deterioration}
Let $\pi'=\pi_{\hat q^\pi, \pi, \Sfix}$.
Assume that for all $s\in\Snotfix$ and $a\in\cA$, $\hat q^\pi(s,a)\approx_\omega q^\pi(s,a)$.  %
Then, for any $s\in\cS$, %
$
v^{\pi'}(s) \ge v^\pi(s)\,.
$
\end{lemma}
\vspace{-1em}
\begin{proof}
Fix any $s\in\cS$.
If $s\in\Sfix$ or $\uboldpi \ge \lbnewpi$, then $\pi'(s)=\pi(s)$ and therefore $q^\pi(s,\pi'(s))=v^\pi(s)$.
Otherwise, $s \not\in\Sfix$ and $\uboldpi\le \lbnewpi$, hence $\pi'(s)=\argmaxa \hat q^\pi(s,a)$, and it follows by our assumptions that
$q^\pi(s,\pi'(s)) \ge \hat q^\pi(s,\pi'(s))-\omega = \lbnewpi>\uboldpi\ge q^\pi(s,\pi(s)) = v^\pi(s)$.
Therefore, in any case, $q^\pi(s,\pi'(s))\ge v^\pi(s)$.
Since this holds for any $s\in\cS$, the Policy Improvement Theorem \citep[Section 4.2]{sutton2018reinforcement} implies that for any $s\in\cS$, $v^{\pi'}(s) \ge v^\pi(s)$. %
\end{proof}

\vspace{-0.3em}

Next we introduce two approximate optimality criterion for a policy on a set of states:
\begin{definition}[Policy optimality on a set of states]\label{def:policy-optimality}
A policy $\pi$ is $\Delta$-optimal (for some $\Delta\ge0$) on a set of states $\cS'\subseteq \cS$, if for all $s\in\cS'$,
$
v^\star(s)-v^{\pi}(s)\le \Delta\,.
$
\end{definition}
\begin{definition}[Next-state optimality on a set of states]\label{def:next-state-optimality}
A policy $\pi$ is next-state $\Delta$-optimal (for some $\Delta\ge0$) on a set of states $\cS'\subseteq \cS$, if for all $s\in\cS'$ and all actions $a\in\cA$,
$\int_{s'\in\cS} \left(v^\star(s')-v^{\pi}(s')\right) \diff P(s'|s,a)  \le \Delta$.
\end{definition}
Note that in the special case of $\cS'=\cS$ the first property implies the second, that is, if $\pi$ is $\Delta$-optimal on $\cS$, then it is also next-state $\Delta$-optimal on $\cS$.
Next, we show that the suboptimality of a policy updated by CAPI evolves as follows (the proof is relegated to \cref{app:proof-of-lem:iteration-progress}):
\begin{lemma}[Iteration progress]\label{lem:iteration-progress}
Let $\pi'=\pi_{\hat q^\pi, \pi, \Sfix}$.
Assume that for all $s\in\Snotfix$ and $a\in\cA$, $\hat q^\pi(s,a)\approx_\omega q^\pi(s,a)$, %
and that $\pi$ is next-state $\Delta$-optimal on $\Snotfix$. %
Then $\pi'$ is $(4\omega+\gamma\Delta)$-optimal on $\Snotfix$.  %
\end{lemma}

\subsection{CAPI guarantee with accurate estimation everywhere}

To obtain a final suboptimality guarantee for CAPI, first consider the ideal scenario
in which we assume that we have a mechanism to estimate $q^\pi(s,a)$ up to some $\omega$ accuracy for all $s\in\cS$ and $a\in\cA$, and for any policy $\pi$:
\begin{assumption}\label{ass:estimate-oracle}
There is an oracle called $\estimate$
that accepts a policy $\pi$ and returns $\hat q^\pi:\cS\times\cA\to\R$ such that
for all $s\in\cS$ and $a\in\cA$, $\hat q^\pi(s,a)\approx_\omega q^\pi(s,a)$. %
\end{assumption}

\begin{theorem}[CAPI performance]\label{thm:capi-perf}
Assume \textsc{CAPI} (\cref{alg:api-variants}) is run with $\Sfix=\{\}$, iteration count to $I=\ceil{\log\omega/\log \gamma}$,
and suppose that the estimation used in Line~\ref{line:api-magic-estimate} satisfies Assumption~\ref{ass:estimate-oracle}.
Then the policy $\pi_I$ returned by the algorithm is $5\omega/(1-\gamma)$-optimal on $\cS$. 
\end{theorem}
\vspace{-1.5em}
\begin{proof}
We prove by induction that policy $\pi_i$ is $\Delta_i$-optimal on $\cS$ for 
$\Delta_i=4\omega\sum_{j\in[i]}\gamma^j + \frac{\gamma^i}{1-\gamma}$.
This holds immediately for the base case of $i=0$, as rewards are bounded in $[0,1]$ and thus $v^\star(s)\le 1/(1-\gamma)$ for any $s$.
Assuming now that the inductive hypothesis holds for $i-1$
we observe that $\pi_{i-1}$ is next-state $\Delta$-optimal on $\cS=\Snotfix$. 
Together with Assumption~\ref{ass:estimate-oracle}, this implies that the conditions of \cref{lem:iteration-progress} are satisfied for $\pi=\pi_{i-1}$, which yields $v^\star(s)-v^{\pi_{i}}(s) \le 4\omega+\gamma\Delta_{i-1}=\Delta_i$, finishing the induction.
Finally, by the definition of $I$, $\pi_I$ is $\Delta_I$-optimal with 
$\Delta_I\le \frac{4\omega}{1-\gamma}+ \frac{\gamma^I}{1-\gamma}
\le \frac{5\omega}{1-\gamma}$. \qedhere
\end{proof}

\section{Local access planning with $q^\pi$-realizability}
\label{sec:planning}

Our planner, \textsc{CAPI-Qpi-Plan}, is based on the \textsc{Confident MC-LSPI} algorithm of \cite{yin2022efficient}.
This latter algorithm gradually builds a \emph{core set} of state-action pairs whose corresponding features are informative. The $q$-values of the state-action pairs in the core set are estimated using rollouts. 
The procedure is restarted with an extended core set whenever the algorithm encounters a new informative feature.
If such a new feature is not encountered, the estimation error can be controlled, and the estimation is extended to all state-action pairs using the least-squares estimator.
Finally, the extended estimation is used in Line~\ref{line:api-magic-estimate} of \textsc{API}.

\textsc{CAPI-Qpi-Plan} improves upon \textsc{Confident MC-LSPI} in two ways. First, using \textsc{CAPI} instead of \textsc{API} improves the final suboptimality bound by a factor of the effective horizon.
Second, we
apply a novel analysis on a more modular variant of the \textsc{ConfidentRollout} subroutine used in \textsc{Confident MC-LSPI}, which delivers $q$-estimation accuracy guarantees with respect to a large class of policies simultaneously.
This allows for a dynamically evolving version of policy iteration, that does not have to restart whenever a new informative feature is encountered.
Intuitively, this prevents duplication of work.

\subsection{Estimation oracle}
\label{sec:estimate}

To obtain an algorithm for planning with local access
whose performance degrades gracefully with the uniform approximation error,
we must weaken \cref{ass:estimate-oracle}.
This is because under local access, we cannot guarantee to cover all states or hope to obtain accurate $q$-value estimates for all states.
Instead, we are interested in an accuracy guarantee that holds for $q$-values only on some subset $\cS'\subseteq \cS$ of states, but holds simultaneously for \emph{any} policy that agrees with $\pi$ on $\cS'$ but may take arbitrary values elsewhere.
For this, we define the extended set of policies:
\begin{definition}\label{def:extend-policy}
Let $\Pidet$ be the set of all stationary deterministic memoryless policies, $\pi\in\Pidet$, and $\cS'\subseteq \cS$.
For $(\pi,\cS')$, we define $\extend{\pi}{\cS'}$ to be the set of policies that agree with $\pi$ on $s\in\cS'$:
\[
\extend{\pi}{\cS'}=\left\{\pi'\in\Pidet\,:\,\pi(s)=\pi'(s) \text{ for all }s\in\cS'\right\}~.
\]
\end{definition}

\begin{algorithm}[t]
\caption{
\textsc{Measure}}\label{alg:measure}
\begin{algorithmic}[1]
\State \textbf{Input:} state $s$, action $a$, deterministic policy $\pi$, set of states $\cS'\subseteq \cS$, accuracy $\omega>0$, failure probability $\zeta\in(0,1]$
\State \textbf{Initialize:} $\Hgamma\gets\ceil{\log((\omega/4)(1-\gamma))/\log \gamma},\,n\gets\ceil{(\omega/4)^{-2}(1-\gamma)^{-2}\log(2/\zeta)/2}$ \label{line:measure-init-set-n}
\For{$i=1$ to $n$}
  \State $(S, R_{i,0})\gets \simulator(s,a)$
  \For{$h=1$ to $\Hgamma-1$}
    \If{$S\not\in\cS'$}
      \Return $(\text{discover}, S)$ \label{line:est-discover}
    \EndIf
    \State $A\gets\pi(S)$ \label{line:eval-policy}
    \State $(S, R_{i,h})\gets \simulator(S,A)$  \Comment{Call to the simulator oracle}
  \EndFor
\EndFor
\State \Return $(\text{success}, \frac1n \sum_{i=1}^n \sum_{h=0}^{\Hgamma-1} \gamma^h R_{i,h})$
\end{algorithmic}
\end{algorithm}

We aim to first accurately estimate $q^\pi(s,a)$ for \emph{some specific} $(s,a)$ pairs, based on which we extend the estimates to other state-action pairs using least-squares.
To this end, we first devise a subroutine called \textsc{Measure} (\cref{alg:measure}).
\textsc{Measure} is a modularized variant of the \textsc{ConfidentRollout} subroutine of \cite{yin2022efficient}.
The modularity of our variant is due to the parameter $\cS'$ that corresponds to the set of states on which the planner is “confident” for \textsc{ConfidentRollout}.
\textsc{Measure} unrolls the policy $\pi$ starting from $(s,a)$ for a number of episodes, each lasting $H$ steps, and returns with the average measured reward.
Throughout, we let $\Hgamma=\ceil{\log((\omega/4)(1-\gamma))/\log \gamma}$ be the effective horizon.
At the end of this process, \textsc{Measure} returns status \emph{success} along with the empirical average $q$-value, where compared to \cref{eq:v-def}, the discounted summation of rewards is truncated to $H$. %
If, however, %
the algorithm %
encounters a state not in its input $\cS'$, 
it returns with status \emph{discover}, along with that state.
This is because in such cases, the algorithm could no longer guarantee an accurate estimation with respect to any member of the extended set of policies.
The next lemma, proved in \cref{app:proof-of-lem:weaker-estimate-oracle}, shows that \textsc{Measure} provides accurate estimates of the action-value functions for members of the extended policy set.

\begin{lemma}\label{lem:weaker-estimate-oracle}
For any input parameters $s\in \cS, a\in \cA,\pi \in \Pidet, \cS' \subset \cS, \omega>0,\zeta \in (0,1)$, \textsc{Measure} %
either returns with $(\text{discover},s')$ for some $s'\not\in\cS'$ (Line~\ref{line:est-discover}),
or it returns with $(\text{success},\tilde q)$ such that with probability at least $1-\zeta$,
\begin{align}\label{eq:estimate-correct}
q^{\pi'}(s,a)\approx_\omega \tilde q \quad \text{ for all } \quad \pi'\in\extend{\pi}{\cS'}. 
\end{align}
\end{lemma}

Suppose we have a list of state-action pairs $C=(s_i,a_i)_{i\in [|C|]}$ and corresponding $q$-estimates $\bar q=(\bar q_i)_{i\in|C|}$.
We use the regularized least-squares estimator $\lse$ (Eq.~\ref{eq:lse-def})
to extend the estimates for all state-action pairs,
with regularization parameter $\lambda=\omega^2/\thetabound^2$ (recall that $\thetabound$ is defined in \cref{def:q-pi-realizable}):
\begin{align}
V(C)&=\lambda \I + \textstyle{\sum_{i \in [|C|]}} \phi(s_i,a_i)\phi(s_i,a_i)^\top  
\,, \label{eq:v-matrix-def}\\
\label{eq:lse-def}
\lse_{C, \bar q}(s,a) &= \ip{\phi(s,a), V(C)^{-1} \textstyle{\sum_{i\in [|C|]}} \phi(s_i,a_i) \bar q_i}\,.
\end{align}
For $C=\bar q=()$ (the empty sequence), we define $\lse_{C,\bar q}(\cdot,\cdot)=0$.
This estimator satisfies the guarantee below. 
\begin{lemma}\label{lem:lse-guarantee}
Let $\pi$ be a stationary deterministic memoryless policy.
Let $C=(s_i,a_i)_{i\in [n]}$ be sequences of state-action pairs of some length $n\in\N$ and $\bar q=(\bar q_i)_{i\in[n]}$ a sequence of corresponding reals such that
for all $i\in[n]$, $q^\pi(s_i,a_i)\approx_\omega \bar q_i$.
Then, for all $s,a\in\cS \times\cA$,
\begin{align}\label{eq:lse-guarantee}
\left|\lse_{C, \bar q}(s,a) - q^\pi(s,a)\right| \le \epsilon+\norm{\phi(s,a)}_{V(C)^{-1}} \left(\sqrt{\lambda} \thetabound + (\omega+\epsilon)\sqrt{n}\right)\,,
\end{align}
where 
$\epsilon$ is the uniform approximation error from \cref{def:q-pi-realizable}.
\end{lemma}
The proof is given in \cref{app:proof-of-lem:lse-guarantee}. 
The order of the estimation accuracy bound (Eq.~\ref{eq:lse-guarantee}) is optimal, as shown by the lower bounds of \citet{Du_Kakade_Wang_Yan_2019} and \citet{LaSzeGe19}.

We intend to use the $\lse$ estimator given in \cref{eq:lse-def} and the bound in \cref{lem:lse-guarantee} only for state-action pairs where 
$\norm{\phi(s,a)}_{V(C)^{-1}}\le 1$ (and $n=\ordot(d)$).
We call these state-action pairs \emph{covered} by $C$, and we call a state $s$ covered by $C$ if for all their corresponding actions $a$, the pair $(s,a)$ is covered by $C$:
\begin{align}
\actioncover(C)&=\{(s,a)\in\cS\times\cA\,:\, \norm{\phi(s,a)}_{V(C)^{-1}} \le 1\} \label{eq:actioncover-def}
\\
\cover(C)&=\{s\in\cS\,:\, \forall a\in\cA,\,(s,a)\in\actioncover(C)\}\,. \label{eq:cover-def}
\end{align}
We will use the parameter $\Sfix$ of CAPI (see CAPI's update rule in Eq.~\ref{eq:capi-pi-update}) to ensure policies are only updated on covered states, where the approximation error is well-controlled by \cref{eq:lse-guarantee}. %

\subsection{Main algorithm}
\label{sec:mainalg}

Finally, we are ready to introduce \textsc{CAPI-Qpi-Plan}, presented in \cref{alg:capi-qpi-plan},
our algorithm for planning with local access under approximate $q^\pi$-realizability.
For this, we define levels $l=0,1,\dots,H$, and corresponding suboptimality requirements: For any $l\in[H+1]$, let
\[
\Delta_l=8(\epsilon+\omega)\left(\sqrt{\tilde d}+1\right)\sum_{j\in[l]}\gamma^{j} + \frac{\gamma^l}{1-\gamma}\,,
\]
for some $\tilde d=\thetat(d)$ defined in \cref{eq:tilde-d-def-and-coreset-bound}.
For each level $l$, the algorithm maintains a policy $\pi_l$ and a set of covered states on which it can guarantee that $\pi_l$ is a $\Delta_l$-optimal policy.
More specifically, this set is $\cover(C_l)$, where $C_l$ is a list of state-action pairs with elements $C_{l,i}=(s_l^i,a_l^i)$ for $i\in[|C_l|]$.
The algorithm maintains the following suboptimality guarantee below, which we prove 
in \cref{app:proof-of-lem:level-suboptimality} 
after showing some further key properties of the algorithm.
\begin{lemma}\label{lem:level-optimality}
Assuming that \cref{eq:estimate-correct} holds whenever \textsc{Measure} returns \emph{success},
$\pi_l$ is $\Delta_l$-optimal on $\cover(C_l)$ (\cref{def:policy-optimality}) for all $l\in[\Hgamma+1]$ 
at the end of every iteration of the main loop of %
\textsc{CAPI-Qpi-Plan}. %
\end{lemma}
\textsc{CAPI-Qpi-Plan} aims to improve the policies, while \emph{propagating} the members of $C_l$ to $C_{l+1}$, and so on, all the way to $C_H$.
During this, whenever the algorithm discovers a state-action pair with a sufficiently “new” feature direction, this pair is appended to the sequence $C_0$ corresponding to level $0$, as there are no suboptimality guarantees yet available for such a state.
However, such a discovery can only happen $\ordot(d)$ times.
When, eventually, all discovered state-action pairs end up in $C_{\Hgamma}$, the final suboptimality guarantee is reached, and the algorithm returns with the final policy.
Note that in the local access setting we consider, the algorithm cannot enumerate the set $\cover(C_l)$, but can answer membership queries, that is, for any $s\in\cS$ it encounters, it is able to decide if $s\in\cover(C_l)$.
The algorithm maintains sequences $\bar q_l$, corresponding to $C_l$, for each level $l$. 
Whenever a new $(s,a)$ pair is appended to the sequence $C_l$, a corresponding $\unknown$ symbol
is appended to the sequence $\bar q_l$, to signal that an estimate of $q^{\pi_l}(s,a)$ is not yet known.

\begin{algorithm}[t]
\caption{
\textsc{CAPI-Qpi-Plan}}\label{alg:capi-qpi-plan}
\begin{algorithmic}[1]
\State \textbf{Input:} initial state $s_0\in\cS$, dimensionality $d$, parameter bound $\thetabound$, accuracy $\omega$, failure probability $\delta>0$
\State \textbf{Initialize:} $\Hgamma\gets\ceil{\log((\omega/4)(1-\gamma))/\log \gamma}$, for $l\in[\Hgamma+1]$, $C_l\gets (),\,\bar q_l\gets (),\,\pi_l\gets\text{ policy that always returns action $\cA_1$}$, $\lambda\gets \omega^2/\thetabound^2$ %
\While{True} \label{line:main-iteration} \Comment{main loop}
  \If{$\exists a\in\cA,\,\ (s_0,a)\not\in\actioncover(C_0)$}  \label{line:init-c-1}
    \State append $(s_0,a)$ to $C_0$, append $\unknown$ to $\bar q_0$  \label{line:append-to-c0}
    \State \Break  \label{line:init-c-2}
  \EndIf
  \State let $\lmin$ be the smallest integer such that $\bar q_{\lmin,m}=\unknown$ for some $m$;  set $\lmin=\Hgamma$ if no such $l$ exists \label{line:set-l} 
  \If{$\lmin =  \Hgamma$}
  \Return $\pi_{\Hgamma}$ \label{line:capi-qpi-alg-return}
  \EndIf
  \State $(\text{status}, \text{result})\gets \measure(s_\lmin^m,a_\lmin^m,\pi_\lmin,\cover(C_\lmin),\omega,\delta/(\tilde d\Hgamma))$  \label{line:measure-qpi-plan} \Comment{recall $C_{\lmin,m}=(s_\lmin^m,a_\lmin^m)$}
  \If{$\text{status}=\text{discover}$}
    \State append $(\text{result}, a)$ to $C_0$ for some $a$ such that $(\text{result}, a)\not\in\actioncover(C_0)$  \label{line:append-new-sa-to-c0}
    \State append $\unknown$ to $\bar q_0$
    \State \Break
  \EndIf
  \State $\bar q_{\lmin,m}\gets \text{result}$
  \If{$\not\exists m' \text{ such that } \bar q_{\lmin,m'}=\unknown$}
    \State $\hat q\gets \lse_{C_\lmin,\bar q_\lmin}$ \label{line:qhat-update}
    \State $\pi'\gets \pi_{\hat q, \pi_\lmin, \cS\setminus\cover(C_\lmin)}$  \label{line:pol-update}
    \State $\pi_{\lmin+1}\gets (s \mapsto \pi_{\lmin+1}(s) \text{ if } s\in\cover(C_{\lmin+1}) \text{ else } \pi'(s))$  \label{line:pol-merge}
    \For{$(s,a)\in C_\lmin$ such that $(s,a)\not\in C_{\lmin+1}$}
     \State append $(s,a)$ to $C_{\lmin+1}$, $\unknown$ to $\bar q_{\lmin+1}$ \label{line:append-to-cl}
    \EndFor
  \EndIf
\EndWhile
\end{algorithmic}
\end{algorithm}

After initializing $C_0$ to cover the initial state $s_0$ (Lines~\ref{line:init-c-1} to \ref{line:init-c-2}),
the algorithm measures $q^{\pi_\lmin}(s,a)$ for the smallest level $\lmin$ for which there still exists a $\unknown$ in the corresponding $\bar q_\lmin$.
After a successful measurement, if there are no more $\unknown$'s left at this level (i.e., in $\bar q_\lmin$), the algorithm executes a policy update on $\pi_\lmin$ (Line~\ref{line:pol-update}) using the least-squares estimate obtained from the measurements at this level, but only for states in $\cover(C_\lmin)$ (using $\Sfix=\cS\setminus\cover(C_\lmin)$).
Next, Line~\ref{line:pol-merge} merges this new policy $\pi'$ with the existing policy $\pi_{\lmin+1}$ of the next level, 
setting $\pi_{\lmin+1}$ to be the policy $\pi''$ defined as
\begin{subnumcases}{
  \pi''(s)=
}
    \pi_{\lmin+1}(s),
    &$\text{if } s\in\cover(C_{\lmin+1})$; \nonumber \\ %
    \pi'(s),
    &$\text{otherwise}$.
       \nonumber %
\end{subnumcases}
This ensures that the existing policy $\pi_{\lmin+1}$ remains unchanged by $\pi''$ (its replacement) on states that are already covered by $C_{\lmin+1}$, and therefore %
$\pi''\in\extend{\pi_{\lmin+1}}{\cover(C_{\lmin+1})}=\extend{\pi''}{\cover(C_{\lmin+1})}$.
We also observe that $C_l$ can only grow for any $l$ (elements are never removed from these sequences), thus for any update where $C_l$ is assigned a new value $C'_l$ (Lines~\ref{line:append-to-c0}, \ref{line:append-new-sa-to-c0}, and \ref{line:append-to-cl}),
$V(C'_l) \mge V(C_l)$, and therefore $\cover(C'_l)\supseteq \cover(C_l)$ and 
$\extend{\pi_l}{\cover(C'_l)}\subseteq \extend{\pi_l}{\cover(C_l)}$. Combining these properties yields the following result:
\begin{lemma}\label{lem:policies-transitively-stay-in-extend}
If for any $l\in[\Hgamma]$, $\pi_l$ and $C_l$ take some values $\pi_l^\text{old}$ and $C_l^\text{old}$ at any point in the execution of the algorithm, then at any later point during the execution,
$\pi_l\in\extend{\pi_l}{\cover(C_l)}\subseteq \extend{\pi_l^\text{old}}{\cover(C_l^\text{old})}$.
\end{lemma}
Any value in $\bar q_l$ that is set to anything other than $\unknown$ will never change again. Since as long as the sample paths generated by \textsc{Measure} in Line~\ref{line:measure-qpi-plan} of \textsc{CAPI-Qpi-Plan} remain in $\cover(C_l)$, their distribution is the same under any policy from $\extend{\pi_l}{\cover(C_l)}$, the $\bar q_l$ estimates are valid for these policies, as well.
Combined with \cref{lem:policies-transitively-stay-in-extend}, we get that the accuracy guarantees of \cref{lem:weaker-estimate-oracle} continue to hold throughout:
\begin{lemma}\label{lem:measure-guarantees-stay}
Assuming that \cref{eq:estimate-correct} holds whenever \textsc{Measure} returns \emph{success}, for any level $l$ and index $m$ such that $\bar q_{l,m} \ne \unknown$,
$q^{\pi'}(s_l^m,a_l^m)\approx_\omega \bar q_{l,m}$\,  for all $\pi'\in\extend{\pi_l}{\cover(C_l)}$ 
throughout the execution of \textsc{CAPI-Qpi-Plan}. %
\end{lemma}

Once $\pi_{\lmin+1}$ is updated in Line~\ref{line:pol-merge}, in Line~\ref{line:append-to-cl} we append to the sequence $C_{\lmin+1}$ all members of $C_\lmin$ that are not yet in $C_{\lmin+1}$, while adding a corresponding $\unknown$ to $\bar q_{\lmin+1}$ indicating that these $q$-values are not yet measured for policy $\pi_{\lmin+1}$.
Thus, whenever all $\unknown$ values disappear from some level $l\in[\Hgamma+1]$, by the end of that iteration $C_{l+1}=C_l$, and hence $\actioncover(C_l)=\actioncover(C_{l+1})$.
Together with the fact that for any $l\in[\Hgamma+1]$, whenever a new state-action pair is appended to $C_l$, an $\unknown$ symbol is appended to $\bar q_l$, we have by induction the following result:
\begin{lemma}\label{lem:cs-are-same-upto-l}
Throughout the execution of \textsc{CAPI-Qpi-Plan},%
after Line~\ref{line:set-l} when $\lmin$ is set,
\[
\actioncover(C_0)=\actioncover(C_1)=\dots =\actioncover(C_\lmin)\,.
\]
\end{lemma}
As a result, whenever the \textsc{Measure} call of Line~\ref{line:measure-qpi-plan} outputs $(\text{discover}, s)$ for some state $s$,
by \cref{lem:weaker-estimate-oracle}, there is an action $a\in\cA$ such that
$(s,a)\not\in\actioncover(C_\lmin)=\actioncover(C_0)$.
This explains why adding such an $(s,a)$ pair to $C_0$ is always possible in Line~\ref{line:append-new-sa-to-c0}.
Consider the $i^\text{th}$ time Line~\ref{line:append-new-sa-to-c0} is executed, and denote $s$ by $s_i$ and $a$ by $a_i$, and $V_i=\lambda \I+ \sum_{t=1}^{i-1} \phi(s_t,a_t)\phi(s_t,a_t)^\top $.
Observe that as $V_i=V(C)$, $(s_i,a_i)\not\in\actioncover(C_0)$ implies $\norm{\phi(s_i,a_i)}_{V_i^{-1}}>1$.
Therefore, $\sum_{t=1}^i \min\{1, \norm{\phi(s_t,a_t)}_{V_t^{-1}}\} = i$, and thus by the elliptical potential lemma \citep[Lemma 19.4]{LaSze19:book}, 
$i\le 2d\log\left(\frac{d\lambda+i\featurebound^2}{d\lambda}\right)$.
This inequality is satisfied by the largest value of $i$, that is, the total number of times \textsc{Measure} returns with \emph{discover}.
Since any element of $C_l$ is also an element of $C_0$ for any $l\in[\Hgamma+1]$, we have that at any time during the execution of \textsc{CAPI-Qpi-Plan},%
\begin{align}
\label{eq:tilde-d-def-and-coreset-bound}
\left| C_l\right| \le 4d\log\left(1+\frac{4\featurebound^2}{\lambda}\right) =: \tilde d = \ordot(d)\,.
\end{align}

When \textsc{CAPI-Qpi-Plan} %
returns at Line~\ref{line:capi-qpi-alg-return} with the policy $\pi_{H}$, it is $\Delta_\Hgamma$-optimal on $\cover(C_\Hgamma)$ by \cref{lem:level-optimality} when the estimates of \textsc{Measure} are correct. Furthermore, $s_0\in\cover(C_0)$ is guaranteed by Lines~\ref{line:init-c-1} to \ref{line:init-c-2}, and hence $s_0\in\cover(C_{\Hgamma})$ by \cref{lem:cs-are-same-upto-l} when the algorithm finishes. Hence, bounding $\Delta_\Hgamma$ using the definition of $\Hgamma$ immediately gives the following result:
\begin{lemma}\label{lem:correctness-of-qpi-plan-alg}
Assuming that \cref{eq:estimate-correct} holds whenever \textsc{Measure} returns \emph{success},
the policy $\pi$ returned by \textsc{CAPI-Qpi-Plan} %
is $\Delta$-optimal on $\{s_0\}$ for
\[
\Delta=9(\epsilon+\omega)\left(\sqrt{\tilde d}+1\right)(1-\gamma)^{-1}=\ordot\left((\epsilon+\omega)\sqrt{d}(1-\gamma)^{-1}\right)\,.
\]
\end{lemma}

To finish the proof of \cref{thm:qplan-main}, we only need to analyze the query complexity and the failure probability 
(i.e., the probability of \cref{eq:estimate-correct} not being satisfied for some \textsc{Measure} call that returns \emph{success}) of \textsc{CAPI-Qpi-Plan}:
\vspace{-0.5em}
\begin{proof}[Proof of \cref{thm:qplan-main}]
Both the total failure probability and query complexity of \textsc{CAPI-Qpi-Plan} %
depend on the number of times \textsc{Measure} is executed, as 
this is the only source of randomness and of interaction with the simulator.
\textsc{Measure} can return \emph{discover} at most $|C_0|$ times, which is bounded by $\tilde d$ by \cref{eq:tilde-d-def-and-coreset-bound}. For every $l\in[\Hgamma]$, \textsc{Measure} is executed exactly once with returning \emph{success} for each element of $C_l$. Hence, by \cref{eq:tilde-d-def-and-coreset-bound} again, \textsc{Measure} returns \emph{success} at most $\tilde d\Hgamma$ times, each satisfying \cref{eq:estimate-correct} with probability at least $1-\zeta=1-\delta/(\tilde d\Hgamma)$ by \cref{lem:weaker-estimate-oracle}. By the union bound, \textsc{Measure} returns \emph{success} in all occasions with probability at least $1-\delta$. Hence
\cref{eq:estimate-correct} holds with probability at least $1-\delta$, which, combined with \cref{lem:correctness-of-qpi-plan-alg}, proves \cref{eq:thm-accuracy}.

Each successful run of \textsc{Measure} executes at most $n\Hgamma$ queries ($n$ is set in Line~\ref{line:measure-init-set-n} of \cref{alg:measure}).
Since $\Hgamma<(1-\gamma)^{-1}\log(4\omega^{-1}(1-\gamma)^{-1})=\ordot((1-\gamma)^{-1})$,
in total \textsc{CAPI-Qpi-Plan} %
executes at most $\ordot\left(d (1-\gamma)^{-4}\omega^{-2}\right)$ queries. As this happens at most $\tilde d\Hgamma$ times, we obtain the desired bound on the query complexity.
\end{proof}

\section{Conclusions and future work}

In this paper we presented \textsc{Confident Approximate Policy Iteration}, a confident version of API, which can obtain a stationary policy with a suboptimality guarantee that scales linearly with the effective horizon $H=\ordot(1/(1-\gamma))$. This scaling is optimal as shown by \citet{scherrer2012}.

CAPI can be applied to local planning with approximate $q^\pi$-realizability (yielding the \textsc{CAPI-Qpi-Plan} algorithm)  to obtain a sequence of policies with successively refined accuracies on a dynamically evolving set of states, resulting in a final, recursively defined policy achieving simultaneously the optimal suboptimality guarantee and best query cost available in the literature.
More precisely, \textsc{CAPI-Qpi-Plan} achieves $\ordot(\epsilon\sqrt{d}H)$ suboptimality, %
where $\epsilon$ is the uniform policy value-function approximation error. %
We showed that this bound is the best (up to polylogarithmic factors) that is achievable by any planner with polynomial query cost.
We also proved that the $\ordot\left(d H^4\epsilon^{-2}\right)$ query cost of \textsc{CAPI-Qpi-Plan} is optimal up to polylogarithmic factors in all parameters except for $H$; whether the dependence on $H$ is optimal remains an open question.

Finally, our method comes at a memory and computational cost overhead, both for the final policy and the planner.
It is an interesting question if this overhead necessarily comes with the API-style method we use (as it is also present in the works of \citealp{scherrer2012,scherrer2014approximate}), or if it is possible to reduce it by, for example, compressing the final policy into one that 
is greedy with respect to some action-value function realized with the features.

\section*{Acknowledgements}
\vspace{-0.3em}
The authors would like to thank Tor Lattimore and Qinghua Liu for helpful discussions.
Csaba Szepesv\'ari gratefully acknowledges the funding from Natural
Sciences and Engineering Research Council (NSERC) of Canada, ``Design.R
AI-assisted CPS Design'' (DARPA)  project and the Canada CIFAR AI Chairs
Program for Amii.

\bibliography{linear_fa}

\ifchecklist
\newpage

\section*{Checklist}

\begin{enumerate}

\item For all authors...
\begin{enumerate}
  \item Do the main claims made in the abstract and introduction accurately reflect the paper's contributions and scope?
    \answerYes{}
  \item Did you describe the limitations of your work?
    \answerYes{}
  \item Did you discuss any potential negative societal impacts of your work?
    \answerNA{}
  \item Have you read the ethics review guidelines and ensured that your paper conforms to them?
    \answerYes{}
\end{enumerate}

\item If you are including theoretical results...
\begin{enumerate}
  \item Did you state the full set of assumptions of all theoretical results?
    \answerYes{}
        \item Did you include complete proofs of all theoretical results?
    \answerYes{}
\end{enumerate}

\item If you ran experiments...
\begin{enumerate}
  \item Did you include the code, data, and instructions needed to reproduce the main experimental results (either in the supplemental material or as a URL)?
    \answerNA{}
  \item Did you specify all the training details (e.g., data splits, hyperparameters, how they were chosen)?
    \answerNA{}
        \item Did you report error bars (e.g., with respect to the random seed after running experiments multiple times)?
    \answerNA{}
        \item Did you include the total amount of compute and the type of resources used (e.g., type of GPUs, internal cluster, or cloud provider)?
    \answerNA{}
\end{enumerate}

\item If you are using existing assets (e.g., code, data, models) or curating/releasing new assets...
\begin{enumerate}
  \item If your work uses existing assets, did you cite the creators?
    \answerNA{}
  \item Did you mention the license of the assets?
    \answerNA{}
  \item Did you include any new assets either in the supplemental material or as a URL?
    \answerNA{}
  \item Did you discuss whether and how consent was obtained from people whose data you're using/curating?
    \answerNA{}
  \item Did you discuss whether the data you are using/curating contains personally identifiable information or offensive content?
    \answerNA{}
\end{enumerate}

\item If you used crowdsourcing or conducted research with human subjects...
\begin{enumerate}
  \item Did you include the full text of instructions given to participants and screenshots, if applicable?
    \answerNA{}
  \item Did you describe any potential participant risks, with links to Institutional Review Board (IRB) approvals, if applicable?
    \answerNA{}
  \item Did you include the estimated hourly wage paid to participants and the total amount spent on participant compensation?
    \answerNA{}
\end{enumerate}

\end{enumerate}

\fi

\newpage

\appendix 

\section{Proof of \cref{lem:iteration-progress}}\label{app:proof-of-lem:iteration-progress}

Take any $s\in\Snotfix$.
\begin{align}
  v^\star(s)-v^{\pi'}(s) 
  &= v^\star(s)-q^{\pi'}(s,\pi'(s)) \nonumber \\
  &= v^\star(s)-q^{\pi}(s,\pi'(s)) + q^{\pi}(s,\pi'(s)) - q^{\pi'}(s,\pi'(s)) \nonumber \\
  &\le v^\star(s)-q^{\pi}(s,\pi'(s))\,, \label{eq:iterationproof1}
\end{align}
where the first equality holds because $\pi'$ is deterministic, and the inequality is true because
\[
q^{\pi}(s,\pi'(s)) - q^{\pi'}(s,\pi'(s)) = \gamma \int_{s'\in\cS} \left(
v^{\pi}(s') - v^{\pi'}(s')\right) \diff P(s'|s,\pi'(s)) \le 0
\]
by Lemma~\ref{lem:no-deterioration}. 
Next observe that
\begin{equation}\label{eq:iterationproof2}
\hat q^{\pi}(s,\pi'(s)) \ge \maxa \hat q^{\pi}(s,a) - 2\omega
\end{equation}
since, as $s\not\in\Sfix$, either $\pi'(s)$ is defined by Case~\ref{eq:capi-pi-update-2} as $\pi'(s)=\argmaxa \hat q^{\pi}(s,a)$ and so $\hat q^{\pi}(s,\pi'(s))=\maxa \hat q^{\pi}(s,a)$, or it is defined by Case~\ref{eq:capi-pi-update-1} in which case $\hat q^{\pi}(s,\pi'(s))=\hat q^{\pi}(s,\pi(s))\ge \maxa \hat q^{\pi}(s,a) - 2\omega$.
Combining \cref{eq:iterationproof1,eq:iterationproof2}, we obtain
\begin{align*}
  v^\star(s)-v^{\pi'}(s) 
  &\le v^\star(s) - \hat q^{\pi}(s,\pi'(s)) + \hat q^{\pi}(s,\pi'(s)) - q^{\pi}(s,\pi'(s)) \\
  &\le v^\star(s) - \hat q^{\pi}(s,\pi'(s)) + \omega \\
  &\le v^\star(s) - \maxa \hat q^{\pi}(s,a) + 3\omega \,,
\end{align*}
where in the first line we added and subtracted $\hat q^{\pi}(s,\pi'(s))$, and the second inequality holds as $\hat q^\pi(s,a)\approx_\omega q^\pi(s,a)$ for $s\not\in\Sfix$ and $a\in\cA$ by the assumptions of the lemma.

We continue by adding and subtracting $\maxa q^{\pi}(s,a)$:
\begin{align*}
  v^\star(s)-v^{\pi'}(s) 
  &\le v^\star(s) - \maxa q^{\pi}(s,a)+\maxa q^{\pi}(s,a)-\maxa \hat q^{\pi}(s,a) + 3\omega \\
  &\le v^\star(s) - \maxa q^{\pi}(s,a) + 4\omega \\
  &= \maxa\left[r(s,a)+\gamma \int_{s'\in\cS}v^\star(s')\diff P(s'|s,a)\right] \\
   & \qquad\qquad  -\maxa\left[r(s,a)+\gamma \int_{s'\in\cS}v^\pi(s')\diff P(s'|s,a)\right] + 4\omega \\
  &\le \maxa\left[\gamma \int_{s'\in\cS}\left(v^\star(s')-v^\pi(s')\right)\diff P(s'|s,a)\right] + 4\omega\\
  &\le 4\omega+\gamma\Delta\,,
\end{align*}
where in the %
fifth line we used that $\pi$ is next-state $\Delta$-optimal by assumption.
\qed

\section{Proof of \cref{lem:weaker-estimate-oracle}}\label{app:proof-of-lem:weaker-estimate-oracle}

For an episode trajectory $\{S_h, A_h, R_h\}_{h\in\N}$,
let $K$ be the smallest positive integer such that $S_K\not\in\cS'$.
For any $i \in \{1,\ldots,n\}$, let $I_i$ denote the indicator of the event that at the $i^{\text{th}}$ iteration of the outer loop of \cref{alg:measure}, the algorithm encounters $S\not\in\cS'$ in Line~\ref{line:est-discover}. Note that $\E_{\pi,s,a}[I_i] = \P_{\pi,s,a}[ 1\le K < \Hgamma ]$. Then, by Hoeffding's inequality (see, e.g., \cite{LaSze19:book}), %
with probability at least $1-\zeta/2$,
\begin{align*}
\left|\P_{\pi,s,a}\left[ 1\le K < \Hgamma \right]-\frac1n \sum_{i=1}^n I_i\right| \le \frac{\omega(1-\gamma)}{4}~.
\end{align*}
\textsc{Measure} only returns \emph{success} if all indicators are zero; therefore, the above inequality implies that if \textsc{Measure} returns \emph{success} then, with probability at least $1-\zeta/2$, we have
\begin{align}\label{eq:probab-leave-certain-small}
\P_{\pi,s,a}\left[ 1\le K < \Hgamma \right] \le \frac{\omega(1-\gamma)}{4}\,.
\end{align}
Recall that if \textsc{Measure} returns $(\text{success},\tilde q)$, then $\bar q=\frac1n \sum_{i=1}^n \sum_{h=0}^{\Hgamma-1} \gamma^h R_{i,h}$.
Since %
\[
0 \le q^{\pi}(s,a) - \E_{\pi,s,a}\sum_{h=0}^{\Hgamma-1} \gamma^h R_h = \E_{\pi,s,a}\sum_{h=\Hgamma}^{\infty} \gamma^h R_h \le \frac{\gamma^{H}}{1-\gamma} \le \frac{\omega}{4}~,
\]
another application of Hoeffding's inequality yields that $q^{\pi}(s,a)$ and $\bar q$ are close with high probability: with probability at least $1-\zeta/2$,
\begin{align}\label{eq:qpi-qbar-close}
\left| q^\pi(s,a)  - \bar q \right|
& =
\left| q^\pi(s,a) - \frac1n \sum_{i=1}^n \sum_{h=0}^{\Hgamma-1} \gamma^h R_{i,h} \right|  \nonumber \\
& \le 
\omega/4+
\left| \E_{\pi,s,a}\sum_{h=0}^{\Hgamma-1} \gamma^h R_h - \frac1n \sum_{i=1}^n \sum_{h=0}^{\Hgamma-1} \gamma^h R_{i,h} \right| \le 
\omega/2\,,
\end{align}
where we also used that the range of the sum of the rewards above for every $i$ is $[0,1/(1-\gamma)]$.

Pick any $\pi'\in\extend{\pi}{\cS'}$.
Observe that for any $s\in\cS$ and $a\in\cA$, the distribution of the trajectory $S_0, A_0,R_0,S_1,A_1,R_1,\dots,A_{K-1},R_{K-1},S_K$ is the same under $\P_{\pi',s,a}$ and $\P_{\pi,s,a}$, as $\pi$ and $\pi'$ select the same actions for states in $\cS'$.
By \crefrange{eq:v-def}{eq:bellman-q},
we can write
\begin{align}\label{eq:qpi-qpiprime-close}
\begin{split}
\left| q^{\pi'}(s,a)-q^\pi(s,a) \right| &= 
\left| \E_{\pi',s,a}\left[ \sum_{t\in[K]} \gamma^t R_t + \gamma^K v^{\pi'}(S_K)\right] 
- \E_{\pi,s,a}\left[ \sum_{t\in[K]} \gamma^t R_t + \gamma^K v^{\pi'}(S_K)\right] \right| \\
&=\left| \E_{\pi,s,a}\left[ \gamma^K\left( v^{\pi'}(S_K) - v^\pi(S_K)\right)\right]\right|
\le \frac{1}{1-\gamma} \E_{\pi,s,a}\left[ \gamma^K \right] \\
&\le
\frac{1}{1-\gamma} \P_{\pi,s,a}\left[ 1\le K < \Hgamma \right]  + \frac{\gamma^\Hgamma}{1-\gamma}
\le \frac{1}{1-\gamma} \P_{\pi,s,a}\left[ 1\le K < \Hgamma \right]  + \omega/4 \,.
\end{split}
\end{align}
Combining \cref{eq:qpi-qbar-close,eq:qpi-qpiprime-close,eq:probab-leave-certain-small}, it follows by the union bound that
if \textsc{Measure} returns with $(\text{success},\tilde q)$, then
with probability at least $1-\zeta$,
\[
\pushQED{\qed} 
\left| q^{\pi'}(s,a)-\bar q \right| \le \left| q^{\pi'}(s,a)-q^\pi(s,a) \right| + \left| q^\pi(s,a) - \bar q \right|
\le \omega\,.\qedhere
\popQED
\]

\section{Proof of \cref{lem:lse-guarantee}}\label{app:proof-of-lem:lse-guarantee}

We start the proof by showing that there exists a $\theta\in\R^d$ such that
\begin{equation}
\label{eq:thetaexists}
\norm{\theta}_2\le \thetabound \text{ and  for all $s\in\cS$ and $a\in\cA$, } q^\pi(s,a)\approx_\epsilon \ip{\theta, \phi(s,a)}.
\end{equation}
For any finite set $W\subseteq \cS\times\cA$, $\max_{(s,a)\in W} |q^\pi(s,a)-\ip{\phi(s,a),\theta'}|$ is a continuous function of $\theta'$, hence it attains its infimum on the compact set $\{\theta'\in\R^d\,:\,\norm{\theta'}_2\le\thetabound\}$. By \cref{def:q-pi-realizable}, this infimum is at most $\epsilon$. Therefore, the compact sets
$\Theta_{s,a}=\{\theta'\in\R^d\,:\,\norm{\theta'}_2\le\thetabound\text{ and }|q^\pi(s,a)-\ip{\phi(s,a),\theta'}|\le \epsilon\}$ are non-empty for all $(s,a)\in\cS\times\cA$, and any intersection of a finite collection of these sets is also non-empty. Therefore, $\bigcap_{(s,a)\in\cS\times\cA} \Theta_{s,a}$ is non-empty by \cite[Theorem 2.36]{rudin1976principles}, and any element $\theta$ of this set satisfies \cref{eq:thetaexists}. For the remainder of this proof, fix such a $\theta$.

For any $i\in[n]$, with a slight abuse of notation, we introduce the shorthand $\phi_i=\phi(s_i,a_i)$, and let $\hat q_i = \ip{\theta, \phi_i}$ and $\xi_i =\bar q_i - \hat q_i$. Note that by the triangle inequality, $|\xi_i| \le 
|\bar q_i - q^\pi(s_i, a_i)| + |q^\pi(s_i, a_i)-\hat q_i | \le \omega+\epsilon$.
Let $\bar \theta=V(C)^{-1} \sum_{i\in [n]} \phi_i \bar q_i$ and $\hat \theta = V(C)^{-1} \sum_{i\in [n]} \phi_i \hat q_i$.

For any $v\in\R^d$ by the Cauchy-Schwarz inequality,
\begin{align*}
\left|\ip{\bar\theta-\theta, v}\right|&\le 
\left|\ip{\hat\theta-\theta,v }\right| + \left|\ip{\bar\theta-\hat\theta, v}\right|
\le 
\norm{v}_{V(C)^{-1}}\norm{\hat\theta-\theta}_{V(C)}
+ \left| \ip{V(C)^{-1}\sum_{i\in[n]} \phi_i\xi_i, v} \right|~.
\end{align*}
To bound the first term on the right-hand side above, observe that %
\begin{align*}
  \norm{\hat\theta-\theta}_{V(C)}
  &=\norm{V(C)^{-1}\left(\sum_{i\in[n]} \phi_i {\phi_i}^\top  \right) \theta  -\theta}_{V(C)}
  =\lambda\norm{\theta}_{V(C)^{-1}}\le  \lambda\norm{\theta}_{\frac1\lambda \I}\le\sqrt{\lambda}\thetabound\,,\\
\end{align*}
where in the last line we used that $V(C)\mge \lambda \I$.

The second term can be bounded as
\begin{align*}
  \left| \ip{V(C)^{-1}\sum_{i\in[n]} \phi_i\xi_i, v} \right|
  & \le \sum_{i\in[n]}  \left| \ip{V(C)^{-1} \phi_i\xi_i, v} \right| \\
   &\le (\omega+\epsilon) \sum_{i\in[n]} \left| \ip{V(C)^{-1}\phi_i,v } \right|\\
  &\le (\omega+\epsilon)\sqrt{n} \sqrt{\sum_{i\in[n]} \left(\ip{V(C)^{-1} \phi_i,v }\right)^2}\\
  &\le (\omega+\epsilon)\sqrt{n} \sqrt{v^\top V(C)^{-1} \left(\sum_{i\in[n]} \phi_i {\phi_i}^\top \right) V(C)^{-1} v +   v^\top V(C)^{-1} \lambda \I V(C)^{-1} v } \\
  &= (\omega+\epsilon)\sqrt{n} \norm{v}_{V(C)^{-1}},
\end{align*}
where the first inequality holds by the triangle inequality, the second by our bound on $|\xi_{i}|$, the third by the Cauchy-Schwartz inequality, and the fourth by the positivity of $\lambda$.
Putting it all together, for any $s\in\cS$ and $a\in\cA$, using the previous bounds with $v=\phi(s,a)$, 
\begin{align*}
\left|\lse_{C, \bar q}(s,a) - q^\pi(s,a)\right| & \le \left|q^\pi(s,a)-\ip{\theta, \phi(s,a)}\right| + \left|\ip{\bar\theta-\theta, \phi(s,a)}\right|  \\
& \le \epsilon+\norm{\phi(s,a)}_{V(C)^{-1}} \left(\sqrt{\lambda} \thetabound + (\omega+\epsilon)\sqrt{n}\right)\,,
\end{align*}
completing the proof. \qed

\section{Deriving next-state optimality of $\pi_\lmin$ for \cref{lem:level-optimality}}

\begin{lemma}\label{lem:lse-accuracy-extend}
Assume that \cref{eq:estimate-correct} holds whenever \textsc{Measure} returns \emph{success}.
At any point of \textsc{CAPI-Qpi-Plan} after Line~\ref{line:qhat-update} is executed, 
for any $\pi''\in\extend{\pi_\lmin}{\cover(C_\lmin)}$,
$s\in\cover(C_\lmin)$, and $a\in\cA$,
\begin{align*}
\left|\hat q(s,a) - q^{\pi''}(s,a)\right|
\le (\omega+\epsilon)(\sqrt{\tilde d}+1)\,.
\end{align*}
\end{lemma}
\begin{proof}
By \cref{lem:measure-guarantees-stay} and \cref{eq:estimate-correct},
$\bar q_{l,m}\approx_\omega q^{\pi''}(C_{{l,m}})$
for all $m\in [|C_\lmin|]$ (recall that $C_{{l,m}}$ is the $m^{\text{th}}$ state-action pair in $C_l$).
Therefore, applying \cref{lem:lse-guarantee} with $q^{\pi''}$, $C_\lmin$ and $\bar q_\lmin$,
as $\hat q = \lse_{C_\lmin,\bar q_\lmin}$,
we get that for any $s\in\cover(C_\lmin)$ and all $a\in\cA$,
\begin{align*}
\left|\hat q(s,a) - q^{\pi''}(s,a)\right|
& \le \epsilon + \|\phi(s,a)\|_{V(C_\lmin)^{-1}} \left(\sqrt{\lambda} B + (\omega+\epsilon)\sqrt{|C_\lmin|}\right) \nonumber \\
& \le (\omega+\epsilon)(\sqrt{\tilde d}+1)\,, 
\end{align*}
where the second inequality holds because $\|\phi(s,a)\|_{V(C_\lmin)^{-1}} \le 1$ since $s\in\cover(C_\lmin)$, $|C_\lmin| \le \tilde d$ by \cref{eq:tilde-d-def-and-coreset-bound}, and the definition of $\lambda$.
\end{proof}

\begin{lemma}\label{lem:piplus-aid-for-lem:level-optimality}
Assume that \cref{eq:estimate-correct} holds whenever \textsc{Measure} returns \emph{success}.
Consider a time when Lines~\ref{line:pol-update} to \ref{line:append-to-cl} of \textsc{CAPI-Qpi-Plan} %
are run and assume that at this time, for all $l\in[\Hgamma+1]$, $\pi_l$ is $\Delta_l$-optimal on $\cover(C_l)$.
Then, $\pi_\lmin$ is next-state $(\Delta_\lmin+4(\omega+\epsilon)(\sqrt{\tilde d}+1)/\gamma)$-optimal on $\cover(C_\lmin)$.
\end{lemma}
\vspace{-1em}
\begin{proof}
Let $\pi_\lmin^+$ be defined as in \cref{eq:pi-combine}.
As $\pi_\lmin^+\in\extend{\pi_\lmin}{\cover(C_\lmin)}$,
by \cref{lem:lse-accuracy-extend}, for any $s\in\cover(C_\lmin)$ and all $a\in\cA$,
\begin{align*}
\left|\hat q(s,a) - q^{\pi_\lmin^+}(s,a)\right|
\le (\omega+\epsilon)(\sqrt{\tilde d}+1)\,.
\end{align*}
Similarly, applying \cref{lem:lse-accuracy-extend} with $\pi_\lmin$ (which trivially belongs to $\extend{\pi_\lmin}{\cover(C_\lmin)}$), we also have 
\begin{align*}%
\left|\hat q(s,a) - q^{\pi_\lmin}(s,a)\right|\le (\omega+\epsilon)(\sqrt{\tilde d}+1)\,.
\end{align*}
Therefore,
\begin{equation}
\label{eq:pi+-pi}
\left|q^{\pi_\lmin^+}(s,a) - q^{\pi_\lmin}(s,a)\right|\le 2(\omega+\epsilon)(\sqrt{\tilde d}+1).
\end{equation}
Since $\pi_{\lmin}$ is $\Delta_{\ell}$-optimal on $\cover(C_\lmin)$ by assumption, this makes $\pi_\lmin^+$ $\Delta$-optimal on $\cover(C_\lmin)$ for 
\begin{equation}
\label{eq:Delta1}
\Delta=\Delta_\lmin+2(\omega+\epsilon)(\sqrt{\tilde d}+1).
\end{equation}

For a trajectory in the MDP, let the random variable $\tau$ be the first time the state is in $\cover(C_\lmin)$:
\[
\tau=\min\{t\in\N\,|\,S_t\in\cover(C_\lmin)\}\,.
\]
Since $\pi_\lmin^+$ agrees with $\pi^\star$ on states not in $\cover(C_\lmin)$, the distribution of the trajectory up to and including $S_\tau$ is the same under both policies, starting from any state $s \in \cS$.
Therefore, for any $s \in \cS$,
\begin{align*}
  v^\star(s)-v^{\pi_\lmin^+}(s) &= 
  \E_{\pi^\star,s} \left[\sum_{t\in\N}\gamma^t R_t\right]
  - \E_{\pi_\lmin^+,s} \left[\sum_{t\in\N} \gamma^t R_t\right] \\
  &= \E_{\pi_\lmin^+,s}\left[\gamma^\tau\left(v^\star(S_{\tau})-v^{\pi_\lmin^+}(S_{\tau})\right)\right]\\
  &\le \Delta\,, %
\end{align*}
as $\gamma^\tau \le 1$ and $\pi_\lmin^+$ is $\Delta$-optimal on $\cover(C_\lmin)$.
That is, $\pi_\lmin^+$ is also $\Delta$-optimal on $\cS$ (with $\Delta$ defined in Eq.~\ref{eq:Delta1}).
Using this, for any $s\in\cover(C_\lmin)$, and $a\in\cA$, we have
\begin{align*}
  \lefteqn{\int_{s'\in\cS} \left(v^\star(s')-v^{\pi_\lmin}(s')\right) \diff P(s'|s,a) } \\ %
  & \le \int_{s'\in\cS} \left(v^\star(s')-v^{\pi_\lmin^+}(s')\right) \diff P(s'|s,a)
  + \left|\int_{s'\in\cS} \left(v^{\pi_\lmin^+}(s')-v^{\pi_\lmin}(s')\right)\diff P(s'|s,a)\right|\\
  &\le  \Delta_\lmin+2(\omega+\epsilon)(\sqrt{\tilde d}+1) + \frac1\gamma\left|q^{\pi_\lmin^+}(s,a)-q^{\pi_\lmin}(s,a)\right| \\
  &\le  \Delta_\lmin+2(\omega+\epsilon)(\sqrt{\tilde d}+1) +2(\omega+\epsilon)(\sqrt{\tilde d}+1)/\gamma \\
  & = \Delta_\lmin+4(\omega+\epsilon)(\sqrt{\tilde d}+1)/\gamma\;,
\end{align*}
where %
the third inequality holds by \cref{eq:pi+-pi}. Therefore $\pi_\lmin$ is next-state $(\Delta_\lmin+4(\omega+\epsilon)(\sqrt{\tilde d}+1)/\gamma))$-optimal on $\cover(C_\lmin)$.
\end{proof}

\section{Poof of \cref{lem:level-optimality}}\label{app:proof-of-lem:level-suboptimality}

\begin{proof}[Proof of \cref{lem:level-optimality}] %
We prove by induction on the iterations of the main loop of \textsc{CAPI-Qpi-Plan} %
the
\emph{inductive hypothesis:} at the start of iteration $i$, for all $l\in[\Hgamma+1]$, $\pi_l$ is $\Delta_l$-optimal on $\cover(C_l)$.
We first observe that after initialization, $C_l$ is the empty sequence for every $l$, so
we can apply \cref{lem:lse-guarantee} with $q^\star$ and empty sequences ($n=0$) to get that 
for any $s\in\cover(())$ and $a\in\cA$,
$q^\star(s,a)\le \epsilon+\sqrt{\lambda}\thetabound=\epsilon+\omega$.
Then, $v^\star(s)\le \epsilon+\omega\le \Delta_l$. Therefore, at initialization, any policy is $\Delta_l$-optimal on $\cover(C_l)$ for any $l\in[\Hgamma+1]$. %

Assuming that the inductive hypothesis holds at the start of some iteration, 
it is left to prove that it continues to hold at the end of the iteration (assuming \cref{eq:estimate-correct} holds whenever \textsc{Measure} returns \emph{success}); this implies that the hypothesis also holds at the start of the next iteration and hence also proves the lemma.
For any $(s,a)$ appended to $C_0$, the inductive hypothesis trivially continues to hold as $\Delta_0=1/(1-\gamma) \ge v^\star(s)$ for any $s\in\cS$ because the rewards are bounded in $[0,1]$.
The only other case in which $C_l$ or $\pi_l$ changes for any $l$ is in Lines~\ref{line:pol-merge} and~\ref{line:append-to-cl}, where the changes happen only for $l=\lmin+1$.

We will use \cref{lem:iteration-progress} to analyze the effect of these updates, thus next we show that the conditions of the lemma are satisfied:

\emph{(a)} In \cref{lem:piplus-aid-for-lem:level-optimality} we show that
$\pi_\lmin$ is next-state $(\Delta_\lmin+4(\omega+\epsilon)(\sqrt{\tilde d}+1)/\gamma)$-optimal on $\cover(C_\lmin)$.
In the proof of the lemma, we introduce a policy in \cref{eq:pi-combine} that acts as $\pi_\lmin$ on states in $\cover(C_\lmin)$, and as an optimal stationary deterministic memoryless policy $\pi^\star$ otherwise:
\begin{equation}
 \label{eq:pi-combine}
  \pi_\lmin^+(s)=
\begin{cases} 
    \pi_\lmin(s)
    &\text{if } s\in\cover(C_\lmin); \\% \label{case:pi-combine-1}\\ 
    \pi^\star(s)
    &\text{otherwise}.
\end{cases}
\end{equation}
Intuitively, this policy corrects $\pi_\lmin$ on the low-confidence states. The  proof of \cref{lem:piplus-aid-for-lem:level-optimality} then uses the fact that this policy
is also $q^\pi$-realizable (\cref{def:q-pi-realizable}) and satisfies
$\pi_\lmin^+\in\extend{\pi_\lmin}{\cover(C_\lmin)}$ to show  (i) that the $q$-values of $\pi_\lmin$ and $\pi_\lmin^+$ are close on the measured state-action pairs (via \cref{lem:measure-guarantees-stay} and \cref{lem:lse-accuracy-extend}); (ii) an optimality guarantee on $\pi_\lmin^+$ for all $s\in\cS$; and, as a consequence, (iii) the next-state optimality of $\pi_\lmin$.

\emph{(b)} Next, to analyze the effect of Line~\ref{line:pol-merge}, 
we introduce hypothetical $q$-approximators $\tilde q_l$ for $l\in[\Hgamma+1]$, defined as follows:
At initialization, %
$\tilde q_l(s,a)=0$ for all $l\in[\Hgamma+1]$, $s\in\cS$, and $a\in\cA$.
It is updated every time after Line~\ref{line:qhat-update} of the algorithm is executed as
\begin{subnumcases}{
\label{eq:tildeq-def}
  \tilde q_\lmin(s,a)\gets
}
    \tilde q_\lmin(s,a)
    &$\text{if } s\in\cover(C_{\lmin+1})$;  \label{eq:tildeq-def-1}\\ %
    \hat q(s,a)
    &$\text{otherwise}$.
        \label{eq:tildeq-def-2}
\end{subnumcases}
In other words, $\tilde q_\lmin$ is only updated to the newly computed $\hat q$ for states that are not in $\cover(C_{\lmin+1})$, and stays unchanged for other states.
We show in \cref{lem:pi-is-an-update} that the new policy that $\pi_{\lmin+1}$ is updated to, which is constructed in two steps (Lines~\ref{line:pol-update}--\ref{line:pol-merge}), can be expressed as the result of a \emph{single} \textsc{CAPI} policy update that uses $\tilde q$:
\[
\pi_{\lmin+1}\gets\pi_{\tilde q_\lmin,\pi_\lmin,\cS\setminus\cover(C_l)}\,.
\] %
We show in \cref{lem:lse-accuracy-tildeq} that $\tilde q_\lmin \approx_{\omega'} q^{\pi_\lmin}$ with $\omega'=(\omega+\epsilon)(\sqrt{\tilde d}+1)$ on $\cover(C_\lmin)$.

By the above, we can apply \cref{lem:iteration-progress} with policy $\pi_\lmin$, $q$-approximation $\tilde q_\lmin$ (with approximation error guarantee $\omega'$ on $\cover(C_\lmin)$, and 
$\Sfix=\cS\setminus \cover(C_\lmin)$
to get that the new value of $\pi_{\lmin+1}$ %
is a $\Delta_{\lmin+1}=(8(\omega+\epsilon)(\sqrt{\tilde d}+1) + \gamma \Delta_\lmin)$-optimal policy on $\cover(C_\lmin)$.
By the end of the loop in Line~\ref{line:append-to-cl}, $\cover(C_{\lmin+1})=\cover(C_\lmin)$,
so $\pi_{\lmin+1}$ is $\Delta_{\lmin+1}$-optimal on $\cover(C_{\lmin+1})$.
This finishes the proof that the inductive hypothesis continues to hold at the end of the iteration, finishing the proof of the lemma.
\end{proof}

\section{Auxiliary results for \cref{lem:level-optimality} about $\tilde q_l$}\label{app:auxiliary-lem:level-suboptimality}

Throughout the execution of \textsc{CAPI-Qpi-Plan}, for $l\in[\Hgamma+1]$, let $\tilde q_l^-$, $\pi_l^-$, $C_l^-$ denote the values of variables $\tilde q_\lmin$, $\pi_\lmin$, $C_\lmin$, respectively, 
at the time when Lines~\ref{line:qhat-update}--\ref{line:append-to-cl} were most recently executed with $\lmin=l$ in a previous iteration of the main loop of \textsc{CAPI-Qpi-Plan}. %
If such a time does not exist, let their values be the initialization values.
Thus, $C_l^-$ may (only) change at the start of some iteration $i$ if 
Lines~\ref{line:qhat-update}--\ref{line:append-to-cl} were executed with $\lmin=l$ in the previous iteration $i-1$.
Observe that whenever this happens, Lines~\ref{line:qhat-update}--\ref{line:append-to-cl} may also change $C_{\lmin+1}$ in iteration $i-1$, and this is the only time $C_{l+1}$ can be changed for any $l\in[\Hgamma]$. After this, at the beginning of iteration $i$, $C_{l+1}$ always has the same elements as $C_l^-$.
Therefore, since it also holds at the initialization of the algorithm, %
we conclude that at the start of each iteration,
\begin{align}\label{eq:cover-clp1-clm}
\cover(C_{l+1})=\cover(C_l^-)\,.
\end{align}

\begin{lemma}\label{lem:lse-accuracy-tildeq}
Assume that \cref{eq:estimate-correct} holds whenever \textsc{Measure} returns \emph{success}. Then,
whenever Line~\ref{line:pol-merge} of \textsc{CAPI-Qpi-Plan} %
is executed, for all %
$s\in\cover(C_\lmin)$ and $a\in\cA$,
\begin{align}\label{eq:lse-accuracy-tildeq}
\left|\tilde q_\lmin(s,a) - q^{\pi''}(s,a)\right|\le (\omega+\epsilon)(\sqrt{\tilde d}+1)\quad\quad\text{for all $\pi''\in\extend{\pi_\lmin}{\cover(C_\lmin)}$}\,.
\end{align}
\end{lemma}
\begin{proof}
We prove this by induction for every time Line~\ref{line:pol-merge} is executed with any value of $\lmin$.
We first observe that after initialization, $C_l$ is the empty sequence for every $l$, so
we can apply \cref{lem:lse-guarantee} with $q^\star$ and empty sequences ($n=0$) to get that 
for any $s\in\cover(())$ and $a\in\cA$,
$q^{\pi''}(s,a)\le q^\star(s,a)\le \epsilon+\sqrt{\lambda}\thetabound=\epsilon+\omega$.
Also, $\tilde q_l(\cdot,\cdot)=0$ at initialization, so \cref{eq:lse-accuracy-tildeq} holds for any value of $\lmin$.

Consider a time when  Line~\ref{line:pol-merge} is executed and assume the inductive hypothesis holds for the previous time Line~\ref{line:pol-merge} was executed with the same value of $\lmin$ (or at the initialization if this is the first time), that is, %
\begin{align*}%
\left|\tilde q_\lmin^-(s,a) - q^{\pi''}(s,a)\right|\le (\omega+\epsilon)(\sqrt{\tilde d}+1)\quad\quad\text{for all $\pi''\in\extend{\pi_\lmin^-}{\cover(C_\lmin^-)},\, s\in\cover(C_\lmin^-)$}\,.
\end{align*}
To prove that the statement now holds for any $s \in \cover(C_{\lmin})$, first consider
any $s\in\cover(C_{\lmin+1})=\cover(C_\lmin^-)$. For such an $s$, by \cref{lem:policies-transitively-stay-in-extend} we have 
that
$\extend{\pi_\lmin}{\cover(C_\lmin)}\subseteq\extend{\pi_\lmin^-}{\cover(C_\lmin^-)}$.
Also, by definition, $\tilde q_\lmin(s,\cdot)=\tilde q_\lmin^-(s,\cdot)$ for $s\in\cover(C_{\lmin+1})$.
Combining with the inductive hypothesis, it follows that \cref{eq:lse-accuracy-tildeq} holds for $s\in\cover(C_{\lmin+1})$.

It remains to show that \cref{eq:lse-accuracy-tildeq} also holds for $s\in\cover(C_\lmin)\setminus\cover(C_{\lmin+1})$. For such an $s$, $\tilde q_\lmin(s,\cdot)=\hat q(s,\cdot)$ by definition, and hence \cref{lem:lse-accuracy-extend} implies that \cref{eq:lse-accuracy-tildeq} holds in this case.

Combining the two cases, it follows that the inductive hypothesis continues to hold when Line~\ref{line:pol-merge} is executed.
\end{proof}

\begin{lemma}\label{lem:pi-is-an-update}
Throughout the execution of \textsc{CAPI-Qpi-Plan}, %
at the start of any iteration, for all $l\in[\Hgamma]$,
\begin{align}
\pi_{l+1}=\pi_{\tilde q_l^-,\pi_l^-,\cS\setminus\cover(C_l^-)}\,. \label{eq:lem:pi-is-an-update}
\end{align}
\end{lemma}
\begin{proof}
We prove this by induction for the start of any iteration.
\cref{eq:lem:pi-is-an-update} holds at the start of the algorithm due to its initialization (because at initialiaztion, $\tilde q_l^-(s,a) = 0$ for all $s,a$, and hence by our tie-breaking rule, the policy on the right-hand side of \cref{eq:lem:pi-is-an-update} always chooses action $\cA_1$, which is the initial policy for $\pi_l$). %

In what follows, we use the fact that for any $q:\cS\times\cA\to\R$, policy $\pi$, and $\Sfix\subseteq\cS$, the 
CAPI policy update $\pi_{q, \pi, \Sfix}$ is
a policy whose value at any $s\in\cS$ only depends on $q(s,\cdot)$, $\pi(s)$, and whether or not $s\in\Sfix$, by definition (Eq.~\ref{eq:capi-pi-update}).
Therefore, for an alternative $q',\,\pi',\,\Sfix'$,
for any $s\in\cS$,
$\pi_{q, \pi, \Sfix}(s)=\pi_{q', \pi', \Sfix'}(s)$
whenever the following three conditions hold: (C1) $q(s,a)=q'(s,a)$ for all $a\in\cA$; (C2) $\pi(s)=\pi'(s)$; and (C3) either both or none of $\Sfix$ and $\Sfix'$ include $s$.

Assume the inductive hypothesis holds at the beginning of some iteration.
Let $\pi''$ be the policy Line~\ref{line:pol-merge} updates $\pi_{\lmin+1}$ to, noting that this is the only place where policies are updated.
All we need to prove is that $\pi''$ is equal to
\[
\tilde \pi=\pi_{\tilde q_\lmin,\pi_\lmin,\cS\setminus\cover(C_\lmin)}\,.
\]

First, for any $s\not\in\cover(C_{\lmin+1})$,
$\pi''(s)=\pi'(s)=\pi_{\hat q,\pi_\lmin,\cS\setminus\cover(C_\lmin)}(s)$ and $\hat q(s,\cdot)=\tilde q_\lmin(s,\cdot)$ by definition.
Hence, $\pi''(s)=\pi_{\hat q,\pi_\lmin,\cS\setminus\cover(C_\lmin)}(s)=\pi_{\tilde q_\lmin,\pi_\lmin,\cS\setminus\cover(C_\lmin)}(s)=\tilde \pi(s)$, as
all of conditions (C1)-(C3) are satisfied for $s$ (C2 and C3 hold trivially).

Next, take any $s\in\cover(C_{\lmin+1})=\cover(C_\lmin^-)$.
Then, by Line~\ref{line:pol-merge}, $\pi''(s)=\pi_{\lmin+1}(s)$.
By the inductive hypothesis, the current value of $\pi_{\lmin+1}$ can be written as 
$\pi_{\tilde q_\lmin^-,\pi_\lmin^-,\cS\setminus\cover(C_\lmin^-)}$.
We prove that this policy takes the same value as $\tilde \pi$ at $s$, by showing conditions (C1)-(C3).
First, by \cref{lem:policies-transitively-stay-in-extend}, %
$\pi_\lmin\in\extend{\pi_\lmin^-}{\cover(C_\lmin^-)}$.
Thus, as $s\in\cover(C_{\lmin}^-)$,
$\pi_\lmin(s)=\pi_\lmin^-(s)$, showing condition (C2). %
Furthermore, as $s\in\cover(C_{\lmin+1})$, by definition,
$\tilde q_\lmin(s,\cdot)=\tilde q_\lmin^-(s,\cdot)$, showing condition (C1). %
Finally, as $s\in\cover(C_{\lmin+1})=\cover(C_{\lmin}^-)\subseteq \cover(C_{\lmin})$,
$s\not\in\cS\setminus\cover(C_\lmin^-)$ and $s\not\in\cS\setminus\cover(C_\lmin)$, showing condition (C3).

Combining the two cases, $\pi''(s)=\tilde \pi(s)$ for any $s\in\cS$, finishing the induction.
\end{proof}

\section{Efficient implementation and proof of \cref{thm:mem-comp-cost}}\label{app:memcomp-impl-thm-proof}

In this section we consider the efficient implementation of \textsc{CAPI-Qpi-Plan} %
in terms of memory and computational costs of both the algorithm itself and the final policy it outputs.

Focusing on the memory cost,
first we can observe that throughout the execution of the algorithm, $C_l$ for all $l\in[\Hgamma+1]$ only stores up to $\tilde d$ unique state-action pairs altogether (cf. \cref{eq:tilde-d-def-and-coreset-bound}), as they use the same pairs; let $W=(s_i, a_i)_{i\in\hat d}$ denote these for some $\hat d \le \tilde d$.
Furthermore, throughout the execution of the algorithm, for any level $l$,
the only features that $\pi_l$ depends on are the features associated with members of $W$.
Storing all these features takes $d\hat d$ memory.
Denote all the policies that \textsc{CAPI-Qpi-Plan} %
constructs in Line~\ref{line:pol-merge}, in order, as $\pi^{(0)},\pi^{(1)},\dots,\pi^{(n-1)}$,
where $n$ is the number of times Line~\ref{line:pol-merge} is executed.
Recall from the proof of \cref{thm:qplan-main} that the number of times \textsc{Measure} returns \emph{success}, which is an upper bounds on $n$, is itself bounded by $\tilde d \Hgamma$, hence $n \le \tilde d \Hgamma$.
Together, Lines~\ref{line:pol-update}-\ref{line:pol-merge} construct a policy that,
for an $s\in\cS$, decides whether the action should be $\argmax_{a\in\cA}\ip{\phi(s,a),\theta}$ for some $\theta$ given by $\lse$ (\cref{eq:lse-def}), or
the value of the policy should be determined by a recursive call to a previously constructed policy, either $\pi_{\lmin+1}$ or $\pi_\lmin$ (through $\pi'$).
Now there exist some $a,b \in [n]$  such that $\pi^{(a)}=\pi_\lmin$ and $\pi^{(b)}=\pi_{\lmin+1}$ before the new policy is constructed in Line~\ref{line:pol-merge}.
To implement the new $\pi_{\lmin+1}$ constructed policy, it is enough therefore to store, in addition to the existing policies, $\theta$ (from $\hat q$), the decision rules, and the indices $a$ and $b$.
The decision rules are fully defined by $\theta$, $C_\lmin$, and $C_{\lmin+1}$.
It is therefore enough to further store $C_\lmin,C_{\lmin+1}\subseteq W$, which can be encoded as $\hat d$-dimensional vectors each, storing the bitmask of which state-action pairs are included. We also store the current value of $\lmin$ (the level) for the newly constructed policy.
Together, a policy thus consumes $3+d+2\hat d$ memory.
We store all policies constructed, along with the features of $W$, and the final value of $V(C_{\Hgamma})^{-1}$, at a memory cost of $d\hat d+\tilde d\Hgamma(3+d+\hat d)+d^2=\ordot(d^2/(1-\gamma))$.
This is the memory cost of the final policy outputted by \textsc{CAPI-Qpi-Plan}. %
The memory cost of running \textsc{CAPI-Qpi-Plan} %
itself is of the same order, as additionally storing $C_l$, $\bar q_l$, and $V(C_l)^{-1}$ for $l\in[\Hgamma+1]$ takes $\ordot(d^2/(1-\gamma))$ memory.

To efficiently implement the final policy found by \textsc{CAPI-Qpi-Plan} with the stored information described above,
we start from evaluating the last policy constructed, $\pi^{(i)}$ for $i=n-1$.
We introduce auxiliary variables $\tilde V(C_l)^{-1}$ and $\tilde C_l$ for $l\in[\Hgamma+1]$ to efficiently track the required values of $V(C_l)^{-1}$ and $C_l$.
We keep updating these variables
so that for $l\in\{\lmin,\lmin+1\}$, they match the values of $V(C_l)^{-1}$ and $C_l$, respectively, at the time of construction of the current policy $\pi^{(i)}$ under consideration, 
where $\lmin$ is the (saved) level of $\pi^{(i)}$.
For $i=n-1$, observe that when it was constructed, $C_0=C_1=\dots =C_{\Hgamma}$ by \cref{lem:cs-are-same-upto-l}.
We therefore start by initializing variables $\tilde V(C_0)^{-1},\dots,\tilde V(C_{\Hgamma})^{-1}$ to the saved final value of $V(C_{\Hgamma})^{-1}$, and variables $\tilde C_0,\dots,\tilde C_{\Hgamma}$ to $W$.
Implementing the decisions of a policy takes an order of $|\cA|d^2$ computation ($|\cA|$ vector and matrix multiplications), after which we recover either the policy output or a previously constructed policy to recurse into.
For the latter case, we have to consider the evaluation of this policy, denoted by $\pi^{(i')}$.
Let the (saved) level of $\pi^{(i')}$ be $\lmin'$.
Before we set $i$ to $i'$ and start evaluating it, we need to update the values of $\tilde V(C_l)$ and $C_l$ for $l\in\{\lmin',\lmin'+1\}$.
The updates are needed for these two levels only, as the decision rule of policy $i'$ only depends on these levels, as shown before.
Let us describe the update procedure for some $l\in\{\lmin',\lmin'+1\}$:
Since $\pi^{(i')}$ was constructed earlier than $\pi^{(i)}$ (i.e., $i'<i$), and $C_{l'}$ can only grow during the algorithm for any $l'\in[\Hgamma+1]$,
we only need to remove members of the variable $\tilde C_l$ to match the value of $C_l$ at the time of construction of $\pi^{(i')}$.
The members to be removed are given by the difference of the members of $\tilde C_l$ and the bitmasks stored for $\pi^{(i')}$ for level $l$.
For each state-action pair $(s,a)$ removed, we also need to update $\tilde V(C_l)^{-1}$ to $\left(\tilde V(C_l)-\phi(s,a)\phi(s,a)^\top \right)^{-1}$, which can be done in order $d^2$ computation using the Sherman–Morrison–Woodbury formula \citep{max1950inverting}.
The total number of such removal operations for any level $l$ is bounded by the sum of the number of state-action pairs in the initialization of $\tilde C_{l'}$ (for $l'\in[\Hgamma+1]$), that is, by $(\Hgamma+1)\hat d$.
As a result, the computational cost of the final policy of \textsc{CAPI-Qpi-Plan} %
is $\ordot((\Hgamma+1)\hat d d^2)+n\ordot(|\cA|d^2)=\ordot(d^3|\cA|/(1-\gamma))$.

Finally, we consider the computational cost of running \textsc{CAPI-Qpi-Plan}. %
The number of iterations of the outer loop is bounded by $\ordot(dH)=\ordot(d/(1-\gamma))$, as each iteration involves either a \textsc{Measure} call that returns \emph{success}, or a new member added to some $C_l$.
For each iteration, Line~\ref{line:init-c-1} takes $\ordot(d^2|\cA|)$, Line~\ref{line:set-l} takes $\ordot(d/(1-\gamma))$, Line~\ref{line:append-new-sa-to-c0} takes $\ordot(d^2|\cA|)$ computation;
for Line~\ref{line:qhat-update}, calculating $\theta$, the second component of the inner product of the least-squares predictor in \cref{eq:lse-def} takes $\ordot(d^2)$ computation, and
if $C_l$ ever changes for some $l$, updating $V(C_l)^{-1}$ by the Sherman–Morrison–Woodbury takes $\ordot(d^2)$ computation.
Overall, all the operations except those associated to the \textsc{Measure} call of Line~\ref{line:measure-qpi-plan} take
$\ordot(d^3|\cA|/(1-\gamma))$ computation in total.
We conclude our calculations by considering the computational cost of the \textsc{Measure} calls, which will dominate the overall computational cost.
Line~\ref{line:est-discover} of \cref{alg:measure} has a computational cost of order $d^2|\cA|$, while the majority of the computational cost comes from evaluating the policy at Line~\ref{line:eval-policy}.
By our previous calculations, this takes $\ordot(d^3|\cA|/(1-\gamma))$ computation and happens (at most) once for each simulator call.
Using the query cost bound of \cref{thm:qplan-main}, we conclude that the computational cost of \textsc{CAPI-Qpi-Plan} %
is $\ordot(d^4|\cA|(1-\gamma)^{-5}\omega^{-2})$.
\qed

\section{Query cost lower bounds with random access}\label{app:lower-bounds}

In this section we prove lower bounds on the worst-case expected query cost of planning algorithms with a simulator supporting \emph{random access}.
Recall from \cref{sec:intro} that in this setting a planner can issue queries for any state-action pair, not just the ones already visited.
As this is a more powerful access to the simulator than \emph{local access}, 
statements that hold for \emph{all} planners using \emph{random access} (as such, all lower bounds presented in this section) trivially hold for planners using \emph{local access}.
We prove two bounds, \cref{theorem:suboptimality-bound} and \cref{theorem:high-prob lower bound}, whose combination trivially implies \cref{thm:main-lower}.

Formally, the planner interacts with a \emph{random access} simulator that simulates some MDP $M$ as follows:
at step $t$ starting from $1$,
given the whole interaction history $H_t = (S_1, A_1, R_1, S_1', \ldots, S_{t-1}, A_{t-1}, R_{t-1}, S_{t-1}')$ (where $H_1$ is the empty sequence by definition),
the planner either selects a state-action pair $(S_t,A_t)$, or halts and outputs a stationary memoryless policy.
The planner is allowed to randomize.
Let $\tau$ denote the number of queries the planner sends to the simulator before it halts, and $\pi_\tau$ the policy it outputs.%
If the planner does not stop, the simulator responds to the query $(S_t,A_t)$ by returning $(S'_t,R_t)$ sampled independently from the transition-reward kernel $\cQ(S_t,A_t)$ of $M$.
Let $\P_M$ denote the probability measure associated with this procedure,
and let $\E_M$ denote the expectation operator corresponding to $\P_M$.
Both $\P_M$ and $\E_M$ implicitly depend on the planner, which is omitted in the notation for brevity but will always be clear from the context.
Using this notation, clearly $\E_M(\tau)$ is the expected query cost of the planner on $M$. 

As usual, we only consider the query complexity of planners which are reasonable in the sense that they can find a near-optimal policies for a class of MDPs:
\begin{definition}[Soundness and query complexity]\label{def:soundness algorithm}
	A planner %
	is said to be $(\alpha, \delta)$-sound for an MDP $M$ if, when used with a simulator of $M$,
	it halts almost surely (i.e., $\P_M \left( \tau < \infty \right) = 1$) and outputs a policy $\pi_\tau$ that is $\alpha$-optimal for $M$ with probability at least $1-\delta$, that is,
	\[
	\P_M\left(v^\star(s_0)-v^{\pi_\tau}(s_0)\le\alpha\right)\ge 1-\delta\,,
	\] 
	where $v^\star$ and $v^{\pi_\tau}$ are the value-functions of the optimal policy and $\pi_\tau$ in the MDP $M$ and $s_0$ is the initial state of $M$.
	A planner is $(\alpha,\delta)$-sound for a class of MDPs $\M$ if it is $(\alpha,\delta)$-sound for every MDP in the class.
	The query complexity of a planner over $\M$ is defined as the maximum of its expected query cost over the members of the class. 
\end{definition}

In the rest of the section, for $d \ge 1$ and $L>0$, we use $\cB_d(L)=\{x \in \R^d: \|x\| \le L\}$ to denote the $d$-dimensional Euclidean ball of radius $L$ centered at the origin.

\subsection{Exponential lower bound for planners with small suboptimality}
We first show an exponential query complexity lower bound for sound planners that guarantee a small suboptimality bound. The result is a simple application of the techniques in \citet{LaSzeGe19}, and establishes the barrier for the suboptimality attainable by query-efficient planners:
\begin{theorem}\label{theorem:suboptimality-bound}
Let $\delta\le0.9$, $\alpha\le 0.49/(1-\gamma)$, and $\epsilon\ge0$, $d \ge 3$.
	There is a class of MDPs $\M$ with uniform policy value-function approximation error $\epsilon$ for some $d$-dimensional feature map
	such that the query complexity of any $(\alpha, \delta)$-sound planner over $\M$ is at least $\exp\left(\Omega\big(d\big(\frac{\epsilon}{\alpha(1-\gamma)}\big)^2\big)\right)$.
\end{theorem}
\vspace{-1em}
\begin{proof}
Our proof is based on a similar complexity lower bound of \citet{LaSzeGe19} for the multi-armed bandit setting, which is a special case of our problem.
As such, we start by rewriting the class of bandit problems they used in their proof in our MDP framework, introducing a set of MDPs $\tiM$ each of which gets into a terminal state with no rewards after the first step.
Let $\alpha'=2.01\alpha(1-\gamma)\le1$ and $k=\floor{\exp\left(\frac{d-2}{8}\left(\frac{\epsilon}{\alpha'}\right)^2\right)}$.
$\tiM=\{\tM_1,\ldots,\tM_k\}$ is defined to be a set of $k$ MDPs as follows:
Each MDP in $\tiM$ has $k$ actions (i.e., $\cA=[k]$) and two states: $\cS=(s_0,s_1)$ with $s_0$ being the initial state, and deterministic transitions
$P(s_1|s,a)=1$ and $P(s_0|s,a)=0$ for all $(s,a) \in \cS \times \cA$.
For any $i\in[k]$, the reward distribution $\cR_i$ for MDP $\tM_i$ is defined as follows:
rewards for state $s_1$ are deterministically zero, that is, $\cR_i(0|s_1,a)=1$ for all $a \in \cA$, making $s_1$ an absorbing state with zero reward, while
rewards for state $s_0$ are deterministically $\alpha'$ for action $i$ and zero otherwise, that is,
$\cR_i(\alpha'|s_0,i)=1$ and $\cR_i(0|s_0,j)=1$ for $j\in[\cA]$ with $j\ne i$. Since this class of MDPs is equivalent to the class of muti-armed bandit problems defined by \citet{LaSzeGe19}, their proof of Corollary~3.3 implies that
\begin{itemize}
\item there exists a feature map $\tphi:\cS \times \cA \to \cB_{d-1} (1)$ such that
$\epsilon$ is the maximum uniform policy value-function approximation error (\cref{def:q-pi-realizable})
over $\tiM$ equipped with features $\tphi$; and
\item any planner that %
almost surely outputs an $\alpha'$-optimal \emph{deterministic} policy for all $\tM \in \tiM$ (when run with a random access simulator for $\tM$)
executes at least 
\begin{equation}\label{eq:simple_exp_lower}
\frac12\exp\left(\frac{d-2}{8}\left(\frac{\epsilon}{\alpha'}\right)^2\right)
\end{equation}
queries in expectation.
\end{itemize}

We construct a new set $\M=\{M_1,\ldots,M_k\}$ of $k$ MDPs 
where for each $i\in[k]$, $M_i$ is a slight modification of $\tM_i$, always returning to the initial state $s_0$ instead of stopping after the first step:
as such, the only modification is that the transition probabilities for all $M \in \M$ are $P(s_0|s,a)=1$ and $P(s_1|s,a)=0$ for all $(s,a) \in \cS\times\cA$. 
Let $\phi:\cS \times \cA \to \cB_{d} (2)$ be the features for all MDPs in $\M$,
where for all $(s,a)\in\cS\times\cA$, $\phi(s,a)$ is a concatenation of the $(d-1)$-dimensional $\tphi(s,a)$ and the scalar $1$, so that the $d^\text{th}$ coordinate of $\phi(s,a)$ is ${\phi(s,a)}_{d}=1$.

Fix any $i\in[k]$ and any stationary deterministic memoryless policy $\pi$, and let $\tilde\theta$ be the parameter realizing the low approximation error for ${\tilde M}_i$ and $\tphi$, that is, satisfying \cref{eq:thetaexists} 
(see \cref{app:proof-of-lem:lse-guarantee} for a proof that such a $\tphi$ exists).
In what follows, we denote $q$- and $v$-functions (with arbitrary superscripts) of an MDP $M$ by adding $M$ as a superscript to the corresponding function.
Let $\theta$ be a concatenation of $\tilde\theta$ and the scalar $\gamma v^\pi_{M_i}(s_0)$.
For any $(s,a)\in\cS\times\cA$,
\[
q^\pi_{M_i}(s,a)=q^\pi_{\tM_i}(s,a)+\gamma v^\pi_{M_i}(s_0)\approx_\epsilon \ip{\tphi(s,a),\tilde\theta}+\gamma v^\pi_{M_i}(s_0)=\ip{\phi(s,a),\theta}\,.
\]
The uniform policy value-function approximation error therefore remains at most $\epsilon$ for $M_i$ with feature map $\phi$, and this is true for any $i\in[k]$.
We can therefore take any $(\alpha, \delta)$-sound planner with query complexity $T$ (for some $T\ge 0$) over $\M$,
and provide it with a simulator of $M_i$ for any $i\in[k]$ (which we can trivially build with access to a simulator of $\tM_i$),
to get a policy $\pi$ that is $\alpha$-optimal for $M_i$ with $\Probab_{M_i}$-probability at least $1-\delta$.
Recall that the rewards of $M_i$ are 0 for every action apart from a single optimal action, $i$, where the reward is $\alpha'$.
Thus, $v^\star_{M_i}(s_0)=\alpha'/(1-\gamma)$ and $v^\pi_{M_i}(s_0)=\alpha'\pi(i|s_0)/(1-\gamma)=\pi(i|s_0) v^\star_{M_i}(s_0)$.
Thus, with probability at least $1-\delta$,
$v^\star_{M_i}(s_0)-v^\pi_{M_i}(s_0) \le \alpha < 0.5 \alpha'/(1-\gamma)=0.5v^\star_{M_i}(s_0)$.
Therefore, $\pi(i|s_0)>0.5$.
As we know that the optimal action achieves a deterministic reward of $\alpha'$, we can test with a single query whether the action that $\pi$ assigns the highest probability to is optimal.
If not, we can run the planner again and repeat the check. Since each run of the planner is successful with probability at least $1-\delta$, independently of each other, almost surely one of the checks eventually passes and we output the deterministic policy that chooses the optimal action. Now the number of times the planner needs to be run is a stopping time (with respect to the sequence of the runs) with expectation at most $1/(1-\delta)$, hence 
the expected query cost of the whole procedure is at most $(T+1)/(1-\delta)$ by Wald's equation. Note that the same policy is $\alpha'$-optimal for $\tM_i$. Therefore, the planner defined above 
almost surely outputs an $\alpha'$-optimal \emph{deterministic} policy for any MDP in $\tiM$, and hence by \cref{eq:simple_exp_lower}
we have
\begin{align*}
T&\ge\frac12(1-\delta)\exp\left(\frac{d-2}{8}\left(\frac{\epsilon}{\alpha'}\right)^2\right)-1\,.
\end{align*}
Therefore $T=\exp\left(\Omega\big(d\big(\frac{\epsilon}{\alpha(1-\gamma)}\big)^2\big)\right)$, finishing the proof.
\end{proof}

\subsection{Lower bound for linear MDPs}
We close this section by proving a lower bounds on the query complexity of \emph{random access} planners for linear MDPs (c.f. \cref{theorem:high-prob lower bound}).

We start by recalling the definition of linear MDPs \citep{zanette2020learning}: %
	An MDP with countable state space is said to be \emph{linear} if there exists a feature map $\phi: \cS \times \cA \to \cB_d (L)$,
	a state-transition feature map $\psi: \cS \to \R^d$, and a reward parameter $\theta_r \in \cB_d (B)$ such that
	$r (s, a) = \ip{\phi (s, a), \theta_r}$ and $P (s' | s, a) = \ip{\phi (s, a), \psi (s')}$ for any $(s, a, s') \in \cS \times \cA \times \cS$, and $\sum_{s\in\cS} \| \psi (s) \|_2 \leq B$.
Clearly, any linear MDP %
satisfies \cref{def:q-pi-realizable} with $\epsilon=0$.
As such, the lower bounds presented below trivially transfer to the $\epsilon$ uniform policy value-function approximation error case for any $\epsilon \ge 0$.

\begin{theorem}\label{theorem:high-prob lower bound}
	Let $\delta \in (0, 0.08]$, $\gamma \in [\frac{7}{12}, 1]$, $H = 1 / (1-\gamma)$,  $\alpha \in (0, 0.05 \gamma H / (1+\gamma)^2]$, and $d \geq 3$.
	Then there is a class of linear MDPs $\M$ such that the query complexity of any $(\alpha, \delta)$-sound planner over $\M$ is at least $\Omega \left( d^2 H^3 / \alpha^2 \right)$.
\end{theorem}

In the remainder of the section we prove the above bound. Throughout we assume that the conditions in \cref{theorem:high-prob lower bound} are satisfied. We start with the construction of the class $\M$ of MDPs, then prove several auxiliary results, before finally presenting the proof of the theorem. %

The construction of $\M$ is based on a combination of hard tabular MDPs  \citep{xiao22curse} and hard linear bandit problems \citep[Section 24.1]{LaSze19:book}.
Each MDP in $\M$ has two states: $\cS=\{s_0,s_1\}$ with $s_0$ being the initial state.
The action space is the intersection of a unit sphere and a $(d-2)$-dimensional hypercube:
$\cA=\{ \pm 1 / \sqrt{d-2} \}^{d-2}$.
We construct MDPs $M_\beta$ for all $\beta\in\cA$, and let $\M=\{M_\beta \,|\,\beta\in\cA\}$.
The feature map $\phi$ is defined, for any $a\in\cA$, as 
\[
\phi (s_0, a) = (1, 0, a^\top)^\top \quad\text{ and } \quad \phi (s_1, a) = (0, 1, 0, \ldots, 0)^\top~.
\]
We define the linear MDPs $M_\beta$ to have deterministic rewards for any $\beta \in \cA$.
Thus, $M_\beta$ is fully defined by its reward parameter $\theta_r$ and state-transition feature map $\psi$, according to the definition of linear MDPs.
Let $\theta_r=(1, 0, \ldots, 0)^\top$, making state $s_0$ the only rewarding state, as then for all $a \in \cA$,
\[
r_\beta(s_0, a) = \ip{\theta_r, \phi(s_0,a)} = 1 \quad \text{and} \quad  r_\beta(s_1, a) = \ip{\theta_r, \phi(s_0,a)} = 0.
\]
Let $\Delta = 4 (1 + \gamma)^2 \alpha / (\gamma H^2)$;
since $\alpha \leq 0.05 \gamma H / (1+\gamma)^2$, $\Delta \leq 0.2 / H=0.2(1-\gamma)$.
Let 
\[
\psi (s_0) = (\gamma, 0, \Delta \beta^\top )^\top \quad \text{and}\quad
\psi (s_1) = (1-\gamma, 1, - \Delta \beta^\top)^\top.
\]
This implies that
\begin{align*}
P_\beta(s_0 | s_0, a) &= \gamma + \Delta \beta^\top a, \hspace{-1cm}&  \qquad &
P_\beta(s_1 | s_0, a) = 1 - \gamma - \Delta \beta^\top a, \\ 
P_\beta(s_0 | s_1, a) &= 0, & \qquad & P_\beta(s_1 | s_1, a) = 1.
\end{align*}
Our assumptions guarantee that $P_\beta$ defines a valid transition kernel with probabilities in $[0,1]$.
The MDP starts in $s_0$ and rewards are collected until the state transitions to $s_1$, which is a terminal state with zero reward.

For the proof, we also need the following notation and supporting lemmas.

\paragraph{Notation.}
The probability measure $\P_{M_\beta}$ induced by the interconnection of a planner and a simulator for $M_\beta$ is written for simplicity as $\P_\beta$.
Similarly, $\E_{M_\beta}$ is written as $\E_\beta$. %
$v_\beta$ (with arbitrary superscripts) denotes value functions (corresponding to the superscripts) of $M_\beta$. 
For any integer $i \in \{1, \ldots, d-2\}$, $\err_i (\pi, \beta) = \sum_{a \in \cA} \pi (a|s_0) I_{\sign{a_i} \neq \sign{\beta_i}}$ denotes the average error of a policy $\pi$ at the $i^\text{th}$ coordinate, where $a_i$ and $\beta_i$ are the $i^\text{th}$ components of $a$ and $\beta$, respectively, and $I_E$ is the indicator function of event $E$.
With a slight abuse of notation, for a stationary memoryless policy $\pi$, we let $\pi^\top \beta$ denote $\sum_{a \in \cA} \pi (a|s_0) a^\top \beta$.

\medskip

\begin{lemma}\label{lemma:state value in hard MDP}
	For any $M_\beta\in\M$, the value function of a stationary memoryless policy $\pi$ is given by
	\begin{align*}
		v^{\pi}_\beta (s_0)
		= \frac{1}{1 - \gamma^2 - \gamma \Delta \pi^\top \beta}\,,
		\quad
		\text{and}
		\quad
		v^\pi_\beta (s_1) = 0\,.
	\end{align*}
\end{lemma}

\begin{proof}
	It clearly holds that $v^\pi_\beta (s_1) = 0$.
	From the Bellman equation, $v^\pi_\beta (s_0) = 1 + \gamma (\gamma + \Delta \pi^\top \beta) v^\pi_\beta (s_0)$,
	and the claim follows from solving this equation for $v^\pi_\beta (s_0)$.
\end{proof}

It is easy to see that the optimal policy in $M_\beta$ is defined by $\pi^\star_\beta(\beta|s_0)=1$ (the actions in $s_1$ do not matter). Hence, by the above lemma,
\begin{align}\label{eq:suboptimality in hard mdp}
	v^\star_\beta (s_0) - v^\pi_\beta (s_0)
	&= \frac{
		\gamma \Delta (1 - \pi^\top \beta)
	}{
		\left(
			1 - \gamma^2 - \gamma \Delta
		\right)
		\left(
			1 - \gamma^2 - \gamma \Delta \pi^\top \beta
		\right)
	}\,.
\end{align}
Because $1 - \pi^\top \beta = 2 \sum_{i=1}^{d-2} \err_i (\pi, \beta) / (d-2)$,
\begin{align}\label{eq:suboptimality in hard mdp with error}
	v^\star_\beta (s_0) - v^\pi_\beta (s_0)
	&= \frac{
		2 \gamma \Delta \sum_{i=1}^{d-2} \err_i (\pi, \beta)
	}{
		(d-2) \left(
		1 - \gamma^2 - \gamma \Delta
		\right)
		\left(
		1 - \gamma^2 - \gamma \Delta \pi^\top \beta
		\right)
	}\,.
\end{align}
Accordingly, to prove a lower bound on the suboptimality of $\pi$, we need a lower bound for the sum of errors, $\sum_{i = 1}^{d-2} \err_i (\pi, \beta)$.
To this end, \cref{lemma:err prob lower bound} below plays a key role.

\begin{lemma}[Error Probability Lower Bound]\label{lemma:err prob lower bound}
	For any planner there exists a $\beta \in \cA$ such that
	\begin{align}
		\sum_{i=1}^{d-2}
		\Probab_{\beta} \left(
			\err_i (\pi_\tau, \beta) \geq \frac{1}{2}
		\right)
		\geq
		\frac{d-2}{2}
		- \frac{d-2}{2} \sqrt{
			1 - \exp \left( - \frac{5 \Delta^2 H \E_\beta [\tau] }{(d-2)^2} \right)
		}\,.
	\end{align}
\end{lemma}

To prove \cref{lemma:err prob lower bound}, we need some technical lemmas.
First, let $\mathcal{F}_t$ for any $t \in \N^+$ denote the $\sigma$-algebra generated by random variables in $H_t$,
with $\mathcal{F}_1$ being the trivial $\sigma$-algebra.
$\mathbb{F} = (\mathcal{F}_t)_{t=1}^\infty$ is chosen to be the filtration.
The following lemma is adopted from Exercise 15.7 of \citet{LaSze19:book} with a slight modification.

\begin{lemma}[KL-divergence decomposition]\label{lemma:KL decomposition}
	Let $M$ and $M'$ be two MDPs differing only in their transition probability kernels, denoted by $P$ and $P'$, respectively.
	Then, for any any $\mathbb{F}$-adapted 
	stopping time $\tau$ satisfying $\Probab_M \left( \tau < \infty \right) = 1$, and an $\mathcal{F}_\tau$-measurable\footnote{By a slight abuse of notation, $\mathcal{F}_\tau$
	is the $\sigma$-algebra generated by the random vector (with random length) $(S_1,A_1,R_1,S'_1,\dots,S_{\tau-1},A_{\tau-1},R_{\tau-1},S'_{\tau-1})$.}
	random variable $Z$,
	\begin{align*}
		\KL \left( \Probab_{M}^Z \middle \|  \Probab_{M'}^Z \right)
		\leq
		\sum_{(s, a) \in \cS \times \cA}
		\E_{M} \left[
			\cN_\tau (s, a)
		\right]
		\KL \left(
			P (\cdot | s, a) \middle \| P' (\cdot | s, a)
		\right)\,,
	\end{align*}
	where
	$\Probab_{M}^Z$ and $\Probab_{M'}^Z$ are the laws of $Z$ under $\Probab_M$ and $\Probab_{M'}$, respectively,
	$\cN_t (s, a)$ denotes the number of queries with $(s, a) \in \cS \times \cA$ up to time step $t$, and $\KL(\cdot,\cdot)$ denotes the Kullback-Leibler (KL-) divergence of two distributions.
\end{lemma}

The next lemma provides an upper bound on the KL-divergence of certain next-state distributions. A similar result appears in the proof of Lemma~6.8 of \citet{zhou2020provably}, but it requires that $\gamma \geq 2/3$; ours only requires the weaker assumption that $\gamma \geq 7/12$.

\begin{lemma}\label{lemma:state-transition prob KL}
	Take any $\beta, \beta' \in \cA$ that only differ at a single coordinate.
	Then for any action $a \in \cA$,
	\begin{align*}
		\KL \left( P_\beta (\cdot | s_0, a) \middle\| P_{\beta'} (\cdot | s_0, a) \right) \leq \frac{5 \Delta^2 H }{(d-2)^2}.
	\end{align*}
\end{lemma}

\begin{proof}
	Our proof relies on Proposition~2 of \citet{xiao22curse}:
	for two Bernoulli distributions $\mathrm{Ber}(p)$ and $\mathrm{Ber}(p')$ with parameters $p,p' \in (0, 1)$, it holds that
	\begin{align*}
		\KL \left(
			\mathrm{Ber} (p) \middle\| \mathrm{Ber} (p')
		\right)
		\le
		\frac{
			( p - p')^2
		}{
			2 \min \left\{ p (1 - p), p' (1-p') \right\}
		}\,.
	\end{align*}
	Since $P_\beta(s_1 | s_0, a) = 1 - \gamma - \Delta \beta^\top a$ and $P_{\beta'}(s_1 | s_0, a) = 1 - \gamma - \Delta (\beta')^\top a$,
	\begin{align}
		\KL \left(
			P_\beta(\cdot | s_0, a) \middle\| P_{\beta'}(\cdot | s_0, a) )
		\right)
		&\le
		\frac{
			\Delta^2 ( (\beta - \beta')^\top a )^2
		}{
			2 \min_{b \in \cA} (\gamma + \Delta \beta^\top b) (1 - \gamma - \Delta \beta^\top b)
		}
		\nonumber \\
		&=
		\frac{
			2 \Delta^2
		}{
			(d-2)^2 \min_{b \in \cA} (\gamma + \Delta \beta^\top b) (1 - \gamma - \Delta \beta^\top b)
		} \label{eq:KL-1}
	\end{align}
	for any action $a \in \cA$. Note that
	\begin{align*}
		\min_{b \in \cA} (\gamma + \Delta \beta^\top b) (1 - \gamma - \Delta \beta^\top b)
		\overset{(a)}{\geq}
		(\gamma + \Delta) (1 - \gamma - \Delta)
		\overset{(b)}{\geq}
		\frac{1 - \gamma - \Delta}{2}
		\overset{(c)}{\geq}
		\frac{2 (1-\gamma)}{5}\,,
	\end{align*}
	where $(a)$ is due to the fact that $x (1-x)$  is monotone decreasing for $x \ge 0.5$ and 
	$\gamma + \Delta \beta^\top b \ge \gamma - \Delta \ge %
	0.5$ since $\gamma \ge 7/12$ and $\Delta \leq 0.2 (1 - \gamma)$,  $(b)$ follows since $0.5 \leq \gamma + \Delta$, and $(c)$ holds because $\Delta \leq 0.2 (1 - \gamma)$.
	Combining this result with \cref{eq:KL-1} concludes the proof of the lemma.
\end{proof}

Now we are ready to prove \cref{lemma:err prob lower bound}.
\begin{proof}[Proof of \cref{lemma:err prob lower bound}]
	Let $\beta^{(i)}$ be a vector obtained by flipping the sign of $\beta$'s $i^\text{th}$ coordinate. Then,
	\begin{align*}
		\lefteqn{\hspace{-3em}\Probab_{\beta} \left(
			\err_i (\pi_\tau, \beta) \geq \frac{1}{2}
		\right)
		+ \Probab_{\beta^{(i)}} \left(
			\err_i (\pi_\tau, \beta^{(i)}) \geq \frac{1}{2}
		\right)}\\
		&=
		\Probab_{\beta} \left(
			\err_i (\pi_\tau, \beta) \geq \frac{1}{2}
		\right)
		+ \Probab_{\beta^{(i)}} \left(
			\err_i (\pi_\tau, \beta) \le \frac{1}{2} 
		\right)
		\\
		&\ge
		\Probab_{\beta} \left(
			\err_i (\pi_\tau, \beta) \geq \frac{1}{2}
		\right)
		+ \Probab_{\beta^{(i)}} \left(
			\err_i (\pi_\tau, \beta) < \frac{1}{2} 
		\right)
		\\
		&\geq
		1 - \sqrt{ 
			1 - \exp \left(
				- \KL \left(
					\Probab_{\beta}^{\err_i (\pi_\tau, \beta)}
					\middle\|
					\Probab_{\beta^{(i)}}^{\err_i (\pi_\tau, \beta)}
				\right)
			\right)
		}
	\end{align*}
	where $\Probab_{\beta}^{\err_i (\pi_\tau, \beta)}, \Probab_{\beta^{(i)}}^{\err_i (\pi_\tau, \beta)} \in \cM_1([0, 1])$ are the laws of the random variable $\err_i (\pi_\tau, \beta)$ in $M_\beta$ and $M_{\beta^{(i)}}$, respectively , and the last line follows from an improved Bretagnolle-Huber inequality (inequality (14.11) of \citet{LaSze19:book}).
	Applying \cref{lemma:KL decomposition,lemma:state-transition prob KL} to the KL-divergence in the exponent in the right hand side of the above inequality
	together with the fact that $\sum_{(s, a) \in \cS \times \cA} \E_\beta \left[ \cN_\tau (s, a) \right] \leq \E_\beta [\tau]$, we can further lower-bound the last line by
	\begin{align*}
		1
		- \sqrt{
			1 - \exp \left(
				- \KL \left( \Probab_{\beta}^{\err_i (\pi_\tau, \beta)} \middle \| \Probab_{\beta^{(i)}}^{\err_i (\pi_\tau, \beta)} \right)
			\right)
		}
		\geq 1 - \sqrt{ 1 - \exp \left( - \frac{5 \Delta^2 H \E_\beta [\tau] }{(d-2)^2} \right) }\,.
	\end{align*}
	Therefore,
	\begin{align*}
		\lefteqn{\hspace{-3em}\frac{1}{|\cA|} \sum_{\beta \in \cA} \sum_{i=1}^{d-2} \Probab_{\beta} \left( \err_i (\pi_\tau, \beta) \geq \frac{1}{2} \right)} \\
		&= \frac{1}{|\cA|} \sum_{i=1}^{d-2} \frac{1}{2}\sum_{\beta \in \cA} \left[ \Probab_{\beta} \left( \err_i (\pi_\tau, \beta)  \geq \frac{1}{2} \right)
		+ \Probab_{\beta^{(i)}} \left( \err_i (\pi_\tau, \beta^{(i)})\geq \frac{1}{2} \right) \right]
		\\
		&\geq \frac{d-2}{2} - \frac{d-2}{2} \sqrt{ 1 - \exp \left( - \frac{5 \Delta^2 H \E_\beta [\tau] }{(d-2)^2} \right) }
	\end{align*}
	where the first equality holds because for any $\beta$, there is exactly one $\beta^{(i)}$ in $\cA$.
	As $\max_{\beta \in \cA} f (\beta) \geq \sum_{\beta \in \cA} f (\beta) / |\cA|$ for any $f: \cA \rightarrow \R$, $\argmax_{\beta \in \cA} \sum_{i=1}^{d-2} \Probab_{\beta} \left( \err_i (\pi_\tau, \beta^{(i)}) \geq 1/ 2 \right)$ satisfies the claim of the lemma.
\end{proof}

Now we are ready to prove \cref{theorem:high-prob lower bound}.
\begin{proof}[Proof of \cref{theorem:high-prob lower bound}]
Take any $(\alpha,\delta)$-sound planner on $\mathcal{M}$.
Let $\err (\pi, \beta) := \sum_{i=1}^{d-2} \err_i (\pi, \beta)$ for brevity.
From \cref{eq:suboptimality in hard mdp with error},
\begin{align}
	\E_\beta \left[
		v^{\star}_\beta (s_0) - v^{\pi_\tau}_\beta (s_0)
	\right]
	&\ge
	\frac{
		2 \gamma \Delta \E_\beta \left[ \err (\pi_\tau, \beta) \right]
	}{
		(d-2) \left(
			1 - \gamma^2 - \gamma \Delta
		\right)
		\left(
			1 - \gamma^2 + \gamma \Delta
		\right)
	} \label{eq:v lower bound}
	\\
	&\ge
	\frac{
		\gamma \Delta
	}{
		(d-2)\left(
			1 - \gamma^2 - \gamma \Delta
		\right)
		\left(
			1 - \gamma^2 + \gamma \Delta
		\right)
	}
	\sum_{i=1}^{d-2} \Probab_\beta \left(
		\err_i (\pi_\tau, \beta) \ge \frac{1}{2}
	\right)\nonumber
	\\
	&\ge
	\frac{
		\gamma \Delta
	}{
		2 \left(
			1 - \gamma^2 - \gamma \Delta
		\right)
		\left(
			1 - \gamma^2 + \gamma \Delta
		\right)
	}
	\left(
		1 - \sqrt{
			1 - \exp \left(
				- \frac{5 \Delta^2 H \E_\beta [\tau] }{(d-2)^2}
			\right)
		}
	\right)\,,\label{eq:expected lower bound with Delta}
\end{align}
where the first inequality holds because $\pi^\top \beta \ge -1$ for any stationary memoryless policy $\pi$, the second inequality is due to the Markov inequality, while the last inequality holds by \cref{lemma:err prob lower bound}.
From \cref{eq:suboptimality in hard mdp with error} and $\pi^\top \beta \le 1$ we also have that
\begin{align}
	\E_\beta \left[
		v^{\star}_\beta (s_0) - v^{\pi_\tau}_\beta (s_0)
	\right]
	&\le
	\frac{
		2 \gamma \Delta \E_\beta \left[ \err (\pi_\tau, \beta) \right]
	}{
		(d-2) \left(
			1 - \gamma^2 - \gamma \Delta
		\right)^2
	}\nonumber
	\\
	&\le
	\frac{
		\gamma \Delta
	}{
		4 \left(
		1 - \gamma^2 - \gamma \Delta
		\right)^2
	}
	\left[
		7 \Probab_\beta \left( \err (\pi_\tau, \beta) > \frac{d-2}{8} \right) + 1
	\right]\,,\nonumber
\end{align}
where the second inequality holds because %
\begin{align*}
	\E_\beta \left[ \err (\pi_\tau, \beta) \right]
	&=
	\E_\beta \left[
	\err (\pi_\tau, \beta) I_{\err (\pi_\tau, \beta) > \frac{d-2}{8}}
	+ \err (\pi_\tau, \beta) I_{\err (\pi_\tau, \beta) \leq \frac{d-2}{8}}
	\right]
	\\
	&\leq
	\E_\beta \left(
		(d-2) I_{\err (\pi_\tau, \beta) > \frac{d-2}{8}}
		+ \frac{d-2}{8} I_{\err (\pi_\tau, \beta) \leq \frac{d-2}{8}}
	\right)
	\\
	&= \frac{d-2}{8} \left(7\Probab_\beta \left( \err (\pi_\tau, \beta) > \frac{d-2}{8} \right) + 1\right)\,.
\end{align*}
Combining this result with \cref{eq:expected lower bound with Delta},
\begin{align*}
	\Probab_\beta \left( \err (\pi_\tau, \beta) > \frac{d-2}{8} \right)
	&\geq
	\frac{2}{7} \frac{1 - \gamma^2 - \gamma \Delta}{1 - \gamma^2 + \gamma \Delta}
	\left(
		1 - \sqrt{
			1 - \exp \left(
				- \frac{5 \Delta^2 H \E_\beta [\tau] }{(d-2)^2}
			\right)
		}
	\right) - \frac{1}{7}
	\\
	&=
	\frac{2}{7} \left(1 - \frac{2 \gamma \Delta}{1 - \gamma^2 + \gamma \Delta} \right)
	\left(
	1 - \sqrt{
		1 - \exp \left(
		- \frac{5 \Delta^2 H \E_\beta [\tau] }{(d-2)^2}
		\right)
	}
	\right) - \frac{1}{7}
	\\
	&>
	\frac{2}{7} \left(1 - \frac{2 \gamma \Delta}{1 - \gamma^2} \right)
	\left(
	1 - \sqrt{
		1 - \exp \left(
		- \frac{5 \Delta^2 H \E_\beta [\tau] }{(d-2)^2}
		\right)
	}
	\right) - \frac{1}{7}\,.
\end{align*}
Note that $\err (\pi_\tau, \beta) > (d-2) / 8$ implies that $v^\star_\beta (s_0) - v^{\pi_\tau}_\beta (s_0) > \alpha$ since similarly to \cref{eq:v lower bound} (i.e., without the expectation)
\begin{align*}
	v^{\star}_\beta (s_0) - v^{\pi_\tau}_\beta (s_0)
	&\ge
	\frac{
		2 \gamma \Delta \, \err (\pi_\tau, \beta)
	}{
		(d-2) \left( (1 - \gamma^2)^2 - \gamma^2 \Delta^2 \right)
	}\nonumber
	>
	\frac{1}{4} \frac{
		\gamma \Delta
	}{
		(1 - \gamma^2)^2 - \gamma^2 \Delta^2
	}
	>
	\frac{1}{4} \frac{
		\gamma \Delta
	}{
		(1 - \gamma^2)^2
	}
	= \alpha\,,\nonumber
\end{align*}
where the last equality follows because $\Delta = 4 (1 + \gamma)^2 \alpha / (\gamma H^2) = 4 (1 - \gamma^2)^2 \alpha / \gamma$.
Therefore,
\begin{align*}
	\Probab_\beta \left( v^\star_\beta (s_0) - v^{\pi_\tau}_\beta (s_0) > \alpha \right)
	& \ge
	\Probab_\beta \left( \err (\pi_\tau, \beta) > \frac{d-2}{8} \right) \\
	&>
	\frac{2}{7} \left(1 - \frac{2 \gamma \Delta}{1 - \gamma^2} \right)
	\left(
	1 - \sqrt{
		1 - \exp \left(
		- \frac{5 \Delta^2 H \E_\beta [\tau] }{(d-2)^2}
		\right)
	}
	\right) - \frac{1}{7}
	\\
	&\overset{(a)}{\geq}
	\frac{2}{7} \left(1 - \frac{0.4 \gamma}{1 + \gamma} \right)
	\left(
	1 - \sqrt{
		1 - \exp \left(
		- \frac{5 \Delta^2 H \E_\beta [\tau] }{(d-2)^2}
		\right)
	}
	\right) - \frac{1}{7}
	\\
	&\overset{(b)}{\geq}
	\frac{8}{35} 
	\left(
	1 - \sqrt{
		1 - \exp \left(
		- \frac{5 \Delta^2 H \E_\beta [\tau] }{(d-2)^2}
		\right)
	}
	\right) - \frac{1}{7}
	\\
	&=
	\frac{3}{35}
	-
	\frac{8}{35} \sqrt{
		1 - \exp \left(
		- \frac{5 \Delta^2 H \E_\beta [\tau] }{(d-2)^2}
		\right)
	}\,.
\end{align*}
where $(a)$ follows since $\Delta \le 0.2(1-\gamma)$,
and $(b)$ follows since $0 \leq 0.4 x / (1 + x) \leq 0.2$ for $x \in [0, 1]$.

This implies that unless $\E_\beta [\tau] \geq \Omega \left( d^2 H^3 / \alpha^2 \right)$, the algorithm is not $(\alpha, \delta)$-sound.
Indeed if
\begin{align*}
	\E_\beta [\tau] \leq \frac{(d-2)^2}{5 \Delta^2 H} \log \left(\frac{1}{1 - \frac{( 3 - 35 \delta)^2}{64}}\right)\,,
\end{align*}
it holds that $\Probab_\beta \left( v^\star_\beta (s_0) - v^{\pi_\tau}_\beta (s_0) > \alpha \right) > \delta$,
contradicting the assumption that the planner is $(\alpha, \delta)$-sound on $\mathcal{M}$ (the upper bound $\delta \le 0.08<3/35$ guarantees that the logarithmic term above is bounded by a constant). This concludes the proof.
\end{proof}

\end{document}